\documentclass{article}
\PassOptionsToPackage{numbers, compress}{natbib}
\usepackage[preprint]{neurips_2026}
 
\usepackage[utf8]{inputenc} 
\usepackage[T1]{fontenc}    
\usepackage{hyperref}       
\usepackage{url}            
\usepackage{booktabs}       
\usepackage{amsfonts}       
\usepackage{nicefrac}       
\usepackage{microtype}      
\usepackage{xcolor}         

\usepackage{titlesec}
\titlespacing{\section}{0pt}{4pt}{2pt}
\titlespacing{\subsection}{0pt}{3pt}{1pt}
\titlespacing{\subsubsection}{0pt}{2pt}{1pt}

\usepackage[pdftex]{graphicx}
\usepackage{amsmath}
\usepackage{amssymb}
\usepackage{mathtools}
\usepackage{amsthm}
\usepackage{amsbsy}
\usepackage{adjustbox,multirow}
\usepackage{wrapfig}
\usepackage{subcaption}
\usepackage{enumitem}
\usepackage[capitalize,noabbrev]{cleveref}
\usepackage{titletoc}

\usepackage{bbm}
\usepackage{mathtools}
\usepackage{mleftright}

\makeatletter
\newcommand{\newreptheorem}[2]{%
    \newtheorem*{rep@#1}{\rep@title}%
    \newenvironment{rep#1}[1]{%
        \def\rep@title{#2 \ref{##1}}%
        \begin{rep@#1}}%
    {\end{rep@#1}}}
\makeatother

\newreptheorem{theorem}{Theorem}

\newreptheorem{corollary}{Corollary}

\newreptheorem{lemma}{Lemma}

\newtheorem{proposition}{Proposition}
\newreptheorem{proposition}{Proposition}

\theoremstyle{definition}

\newreptheorem{definition}{Definition}

\newtheorem{assumption}{Assumption}

\crefname{equation}{}{}
\crefname{section}{Section}{Sections}
\crefname{lemma}{Lemma}{Lemmas}
\crefname{assumption}{Assumption}{Assumptions}
\crefname{algorithm}{Algorithm}{Algorithms}
\crefname{theorem}{Theorem}{Theorems}
\crefname{proposition}{Proposition}{Propositions}
\crefname{corollary}{Corollary}{Corollaries}
\crefname{insight}{Insight}{Insights}
\crefname{definition}{Definition}{Definitions}
\crefname{figure}{Figure}{Figures}
\crefname{table}{Table}{Tables}

\DeclareMathOperator{\E}{\mathbb{E}}

\newcommand{\R}{\mathbb{R}}

\let\P\undefined

\newcommand{\P}{\mathbb{P}} 

\def\ddefloop#1{\ifx\ddefloop#1\else\ddef{#1}\expandafter\ddefloop\fi}
\def\ddef#1{\expandafter\def\csname bb#1\endcsname{\ensuremath{\mathbb{#1}}}}
\ddefloop ABCDEFGHIJKLMNOPQRSTUVWXYZ\ddefloop
\def\ddefloop#1{\ifx\ddefloop#1\else\ddef{#1}\expandafter\ddefloop\fi}
\def\ddef#1{\expandafter\def\csname b#1\endcsname{\ensuremath{\mathbf{#1}}}}
\ddefloop ABCDEFGHIJKLMNOPQRSTUVWXYZ\ddefloop
\def\ddef#1{\expandafter\def\csname c#1\endcsname{\ensuremath{\mathcal{#1}}}}
\ddefloop ABCDEFGHIJKLMNOPQRSTUVWXYZ\ddefloop
\def\ddef#1{\expandafter\def\csname h#1\endcsname{\ensuremath{\hat{#1}}}}
\ddefloop ABCDEFGHIJKLMNOPQRSTUVWXYZ\ddefloop
\def\ddef#1{\expandafter\def\csname wh#1\endcsname{\ensuremath{\widehat{#1}}}}
\ddefloop abcdefghijklmnopqrstuvwxyz\ddefloop
\def\ddef#1{\expandafter\def\csname hc#1\endcsname{\ensuremath{\widehat{\mathcal{#1}}}}}
\ddefloop ABCDEFGHIJKLMNOPQRSTUVWXYZ\ddefloop
\def\ddef#1{\expandafter\def\csname t#1\endcsname{\ensuremath{\widetilde{#1}}}}
\ddefloop ABCDEFGHIJKLMNOPQRSTUVWXYZ\ddefloop
\def\ddef#1{\expandafter\def\csname tc#1\endcsname{\ensuremath{\widetilde{\mathcal{#1}}}}}
\ddefloop ABCDEFGHIJKLMNOPQRSTUVWXYZ\ddefloop

\usepackage{bm}
\makeatletter
\newcommand{\vect}[1]{%
  \ifcat\noexpand#1\relax
    \bm{#1}
  \else
    \mathbf{#1}
  \fi
}
\makeatother

\newcommand{\zeros}{\vect{0}}
\usepackage{prettyref}
\usepackage{xcolor}

\newcommand*\diff{\mathop{}\!\mathrm{d}}

\newcommand{\kl}[2]{\mathrm{D}_\mathrm{KL}\left(#1 \ \middle\| \ #2\right)}

\newcommand{\data}{\mathcal{D}}

\newcommand{\indep}{\perp\!\!\!\!\perp} 

\makeatletter
\DeclarePairedDelimiter{\absplain}{\lvert}{\rvert}
\DeclareRobustCommand{\abs}{\@ifstar{\absplain}{\absplain*}}
\makeatother
\DeclareRobustCommand{\lpar}{\mleft(}
\DeclareRobustCommand{\rpar}{\mright)}
\DeclareRobustCommand{\lbar}{\mleft[}
\DeclareRobustCommand{\rbar}{\mright]}
\DeclareRobustCommand{\lcbar}{\mleft\{}
\DeclareRobustCommand{\rcbar}{\mright\}}

\newcommand{\xhdr}[1]{\textbf{#1.}}

\newcommand{\normal}[2]{\mathcal{N}(#1, #2)}

\newcommand{\obs}{\mathrm{\vect{O}}}

\newcommand{\bernoulli}{\mathrm{Bernoulli}}
\newcommand{\sigmoid}{\mathrm{sigmoid}}
\newcommand{\softargmax}{\mathrm{softmax}}
\newcommand{\uniform}{\mathrm{Unif}}

\newcommand{\lowerbound}{\ell}
\newcommand{\upperbound}{u}

\definecolor{orangesignal}{HTML}{FE6100}
\definecolor{highlightsignal}{HTML}{FFB000}
\definecolor{purplesignal}{HTML}{785ef0}
\definecolor{bluesignal}{HTML}{648fff}
\definecolor{pinksignal}{HTML}{dc267f}

\newcommand{\stepindicator}[1]{%
  {\textbf{\textit{(#1)}}}%
}

\definecolor{citationcolor}{HTML}{648fff} 
\definecolor{linkcolorcustom}{HTML}{648fff}

\hypersetup{ 
    colorlinks=true,       
    linkcolor=pinksignal,
    citecolor=citationcolor, 
    urlcolor=linkcolorcustom 
}

\newcommand{\method}{IV-ICL}

\title{\method: Bounding Causal Effects with Instrumental Variables via In-Context Learning}

\author{%
  Vahid Balazadeh\thanks{Correspondence to \url{vahid@cs.toronto.edu}}\;\,{}$^{1}$$^{2}$ \ \ \ Hamidreza Kamkari$^{3}$ \ \ \ Medha Barath$^{1}$  \\ \textbf{Ricardo Silva}$^{4}$ \ \ \ \textbf{Rahul G. Krishnan}$^{1}$$^{2}$\\
  {}$^1$ University of Toronto\quad {}$^2$ Vector Institute \quad 
  {}$^3$ MIT CSAIL \quad {}$^4$ University College London
}

\begin{document}

\maketitle

\begin{abstract}
The instrumental-variables (IV) setting is standard for partial identification of causal effects when unobserved confounding makes point identification impossible. Existing approaches face methodological bottlenecks: closed-form bound estimands are required---e.g., Balke-Pearl equations in binary IV---and even when available, designing accurate estimators requires manual effort tailored to each estimand. While direct Bayesian inference of the causal effects, instead of the bounds, circumvents these challenges, it is often computationally intensive and suffers from high prior sensitivity or under-dispersed posteriors. As a remedy, we introduce \method, an \emph{amortized} Bayesian in-context learning method that learns the marginal posterior distribution of the causal effects directly and derives bounds as its quantiles. Unlike standard variational inference that optimizes exclusive KL divergence, amortized Bayesian inference minimizes the expected \emph{inclusive} KL, a mass-covering objective. We empirically observe that optimizing inclusive KL can recover the entire identified set across diverse data-generating processes, while exclusive-KL (e.g.\ with variational inference) of the same Bayesian formulation collapses onto a single mode and fails to cover the identified set. We evaluate \method\ on synthetic and semi-synthetic IV benchmarks and show it produces intervals that are more reliably valid and more informative compared to efficient semi-parametric, Bayesian, and plug-in baselines, at $20$--$500\times$ lower inference time. Beyond methodology, we propose a procedure to convert randomized controlled trials into IV benchmarks with provably preserved ground-truth causal effects that enables a more realistic evaluation of partial-identification methods.
\end{abstract}

\section{Introduction}

Estimating causal effects from observational data is often based on untestable assumptions about the data-generating process (DGP)~\citep{pearl2009causality,imbens2015causal}. The most common assumption of no unobserved confounding, that all non-causal paths between the treatment and the outcome are blocked by observed variables~\citep{rubin1974estimating,vanderWeele2013OnTD}, rarely holds in practice, e.g., in randomized experiments with non-compliance~\citep{angrist1995identification,angrist2009mostly,cinelli2025challenges}. When it fails, the causal effects are no longer point-identified, motivating partial identification~\citep{manski2003partial,richardson2015nonparametric}. The goal of partial identification is to characterize the \emph{set} of causal effects compatible with the observed distribution under weaker structural assumptions; the standard instrumental variable (IV) model is the most prominent example~\citep{wright1928tariff,reiersol1945confluence,baker1994paired,angrist1996identification,balke1997bounds,swanson2018partial}.

\xhdr{Two routes for partial identification} A first route, established  for the binary IV setting by  \citet{balke1994counterfactual,balke1997bounds}, derives closed-form bound estimands as functionals of the observed distribution by solving a linear program. Generalizations to discrete or continuous variables~\citep{tian2000probabilities,dawid2003causal,zhang2021bounding,zhang2021non,sachs2023general,duarte2024automated} or to graphs with covariates~\citep{cai2007non,tsiatis2008covariate,long2013sharpening,levis2025,whitehouse2025inference} follow the same pattern but require re-solving the optimization in each new setting (\textbf{L1}).\footnote{Throughout, we use (\textbf{L1})--(\textbf{L5}) to label key limitations of various methods for partial identification across the literature.} Even when a closed-form estimand is available, estimating it from finite data is expensive: the bound estimands are non-smooth (maxima or minima of linear expressions are functions with sharp edges), and obtaining low-variance semi-parametric estimators in presence of high-dimensional covariates requires manually designing estimators~\citep{cai2007non,levis2025,whitehouse2025inference}~(\textbf{L2}).

A second route forgoes closed-form bound estimands and instead infers the causal effects directly via Bayesian inference~\citep{chickering1996clinician,gustafson2005model,greenland2009relaxation,gustafson2010bayesian}: the DGPs themselves are treated as random variables with a wide-prior, the marginal posterior distribution of the causal effects spreads mass across the entire identified set, and its quantiles serve as natural estimators of the lower and upper bounds. This avoids \textbf{L1} and \textbf{L2} entirely. However, classical Bayesian inference via Markov chain Monte Carlo (MCMC) or variational inference has three primary issues. First, along the unidentifiable directions of the parameter space where the likelihood is flat, MCMC mixes poorly and variational inference under-disperses (\textbf{L3}). Second, estimates derived from the posterior are sensitive to the choice of prior, so prior's support over DGPs must be wide enough to span the identified set (\textbf{L4})~\citep{silva2016causal,jesson2020identifying}. Finally, inference must be re-run from scratch per dataset, making large-scale IV analyses infeasible (\textbf{L5}).

\xhdr{Amortized Bayesian inference} Recent work suggests that one way to cut the cost of Bayesian inference is amortized learning with in-context models: pre-train a single transformer on a wide library of synthetic DGPs, and perform inference on new datasets in a single forward pass. Prior-data Fitted Networks (PFNs) are a successful instantiation of this idea~\citep{müller2023transformers,hollmann2025accurate,ma2025tabdpt,qu2025tabicl}. This paradigm has proven successful for causal inference under \emph{point-identifying} assumptions~\citep{balazadeh2025causalpfn,robertson2025pfn,ma2025foundation}. However, extending it to partially identifiable settings requires more work. Flat likelihoods in the unidentifiable directions of the parameter space make the marginal posterior over causal effects prior-sensitive. The goal is no longer to estimate a single effect, but the entire range of feasible causal effects. One can resort to amortizing the marginal posterior over the point-identified lower/upper bound estimands, such as closed-form Balke-Pearl equations~\citep{balke1997bounds}, but this re-introduces limitation \textbf{L1}. 

\xhdr{PFNs meet partial identification} We observe that the PFN training objective is equivalent to minimizing the \emph{inclusive} KL between the true marginal posterior and the model's predicted distribution~\citep{müller2023transformers,nagler2023statistical}. Since inclusive KL is mass-covering, minimizing it pushes the model to spread density wherever the true marginal posterior does. This highlights the challenge in the design of a valid prior for the IV setting: inclusive KL only delivers mass-coverage with respect to the prior, so a partial-identification PFN is only as informative as the breadth of the prior's support over DGPs; we therefore design a prior whose induced causal effect distribution spans a wider attainable range. On such a prior, the trained posterior avoids the under-dispersion that breaks exclusive KL variational inference, so its quantiles are usable as bounds. Building on these observations, we introduce \method, a Bayesian method for partial identification that resolves \textbf{L3}--\textbf{L5} via amortized in-context learning with a wide-range prior and a mass-covering training objective (\cref{fig:main}). Our contributions: 
\begin{figure*}
    \centering
    \captionsetup{font=footnotesize}
    \includegraphics[width=0.85\linewidth]{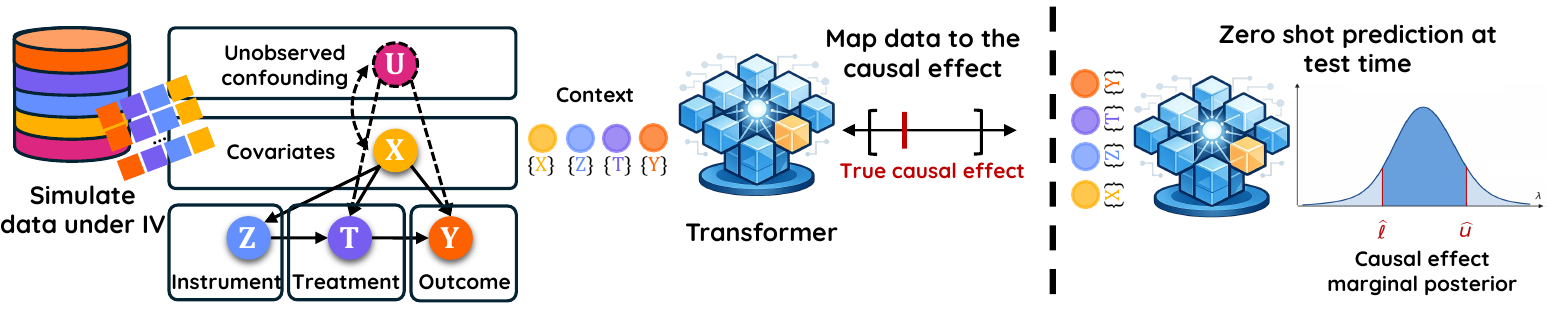}
    \caption{\textbf{\method\ pipeline.} During pre-training, we generate a diverse library of synthetic datasets with IV structure and known ground-truth causal effects, and train a transformer to map each dataset to the marginal posterior of its causal effect. At inference, the trained model takes any new IV dataset as context and returns the causal effect marginal posterior in a single forward pass; bounds are read off as marginal posterior quantiles.}
    \label{fig:main}
    \vspace{-1.5em}
\end{figure*}
\begin{itemize}[leftmargin=*]
    \item A method that does not require solving linear programs to derive the bounds (\textbf{L1}) and amortizes inference into a single computationally cheap forward pass per dataset (\textbf{L2}, \textbf{L5}).
    \item A general-purpose prior over IV-consistent synthetic DGPs, designed to have wide support which allows for better transfer and generalization at deployment (\textbf{L4}).
    \item A theoretically justified procedure to convert data from randomized controlled trials (RCT) into IV benchmarks with preserved ground-truth causal effects, allowing for a realistic IV evaluation.
    \item Empirical evidence, on synthetic and RCT-derived benchmarks, that \method\ produces tighter and more reliably valid intervals (\textbf{L3}) than efficient semi-parametric, Bayesian, and plug-in baselines. 
\end{itemize}
\section{Related Work}
 \xhdr{Bayesian Partial Identification}
Bayesian inference has been used as a framework for partial identification~\citep{rubin1978bayesian,chickering1996clinician,silva2016causal,franks2020flexible,jesson2020identifying}. One strategy is to follow transparent parameterization of the potential outcome model to distinguish between identifiable from non-identifiable parts~\citep{richardson2011transparent,silva2016causal}; this separates the epistemic uncertainty from the uncertainty inherent in the unidentifiable set but requires per-problem manual effort. The other strategy directly learns the marginal posterior of causal effects~\citep{gustafson2005model,greenland2009relaxation,gustafson2010bayesian}, but historically has been limited by under-dispersed variational inference posteriors (\textbf{L3}), prior sensitivity (\textbf{L4}), and computational cost (\textbf{L5}). \method\ pursues the latter strategy and aims to resolve \textbf{L3}--\textbf{L5} via amortized in-context learning with a wide-range prior and a mass-covering training objective.
 
\xhdr{Covariate-Adjusted Bounds}
Incorporating covariates can sharpen IV bounds by exploiting conditional independence structures~\citep{cai2007non}. Recent work has developed frequentist estimators for covariate-adjusted bounds: \citet{levis2025} derive nonparametrically efficient estimators using smooth approximations to the non-differentiable Balke-Pearl functionals, while \citet{whitehouse2025inference} develop general methodology for inference on irregular functionals via smoothing. Both approaches operate on closed-form Balke-Pearl estimands (\textbf{L1}) and require careful nuisance estimation under high-dimensional covariates (\textbf{L2}). Our work bypasses \textbf{L1} entirely by not relying on closed-form bounds and addresses \textbf{L2} through amortized inference.
 
\xhdr{Automated and Optimization-Based Approaches}
General partial identification problems can be formulated as constrained optimization over response functions. For instance, \citet{duarte2024automated} develop automated methods for discrete settings via linear programming.
The computational complexity of such methods, however, grows rapidly with the number of variable categories. A parallel line of work optimizes over distributions of response functions directly using deep generative or neural-network parameterizations~\citep{kilbertus2020class,xia2022neural,balazadeh2022partial,padh2023stochastic,schweisthal2024learning}. While flexible, these methods solve complex and costly optimization problems per dataset (\textbf{L5}).
 
\xhdr{Amortized Inference for Causal Effects}
Recent work leverages transformers for amortized causal inference~\citep{zhang2023towards,bynum2025black,mahajan2025amortized}. CausalPFN~\citep{balazadeh2025causalpfn} demonstrated that pre-training on synthetic data enables causal effect estimation under unconfoundedness, Do-PFN~\citep{robertson2025pfn} handles interventional queries, and \citet{ma2025foundation} developed foundation models for front-door estimation and IV regression under the conditional mean model. These amortized methods learn point-identified estimands and are not amenable to partial identification. In our setting, the objective is to learn an identified \emph{set}, not a point.

\xhdr{Applications of Binary Instrumental Variables} Binary IVs with bounded outcome---the setting we use in this work---appear across a wide range of fields: lift studies in digital advertising~\citep{johnson2017ghost}, randomized encouragement designs in healthcare~\citep{hirano2000assessing}, lottery-based admissions in education~\citep{abdulkadirouglu2011accountability}, and clinical trials with non-compliance~\citep{angrist1996identification}. Many of these settings involve large number of analyses, where per-dataset inference cost becomes a bottleneck (\textbf{L5}).
\section{Background}
\label{sec:background}
\xhdr{Notation and Assumptions} We adopt the potential outcome framework \citep{rubin2005causal}. Assume we observe a dataset $\data$ of $N$ i.i.d. samples from random variables $\vect O = \lpar \vect X, Z, T, Y \rpar$, where $\vect X \in \R^d$ denotes the vector of observed covariates, $Z \in \{0, 1\}$ is the binary instrument, $T \in \{0, 1\}$ denotes the binary treatment, and $Y \in [0, 1]$ is the bounded outcome (binary or continuous). For $z, t \in \{0, 1\}$, we let $T(z)$ be the potential treatment under intervention $Z \coloneqq z$. We further denote the potential outcome under intervention $T \coloneqq t$ by $Y(t)$. Finally, we define $Y(z, t)$ as the potential outcome under intervention $Z \coloneqq z, T \coloneqq t$.\footnote{We focus on the binary setting as most mature methods and evaluation benchmarks exist only for this setting.} Our goal is to estimate the sample average treatment effect (SATE):
\begin{equation}
    \label{eq:sample-ate}
    \lambda(\data) \coloneqq \frac{1}{N} \sum_{i=1}^N \E\lbar Y(t=1) - Y(t=0) \mid \vect X = \vect x_i \rbar.
\end{equation}
We generally cannot rule out the existence of unobserved variables that can confound the causal effect of the treatment on the outcome. In such settings, the SATE is not point-identified, i.e., it cannot be uniquely determined from the data alone, as multiple causal effects are compatible with the same observed distribution~\citep{rubin1974estimating,pearl2009causality}. We thus turn to partial identification: estimating an informative set interval $[\lowerbound, \upperbound]$ that contains the true SATE. We specifically focus on the standard IV setting, which can provably lead to nontrivial values of $\lowerbound$ and $\upperbound$~\citep{balke1994counterfactual,balke1997bounds, manski1990nonparametric}: 
\begin{assumption}[Standard Instrumental Variable] \label{assump:iv} \stepindicator{i} \textbf{Consistency}: Observed variables match their relevant potential variables, i.e., $Y = Y(z, t)$ for all $Z = z$ and $T = t$, and $T = T(z')$ for all $Z = z'$. \stepindicator{ii} \textbf{Positivity}: All covariate strata have positive probability of each instrument value, so conditioning on the instrument $Z$ is well-defined, i.e., $\P\lpar Z = 1 \mid \vect X \rpar \in [\epsilon, 1 - \epsilon]$ for some $ \epsilon \in (0, 0.5)$. \stepindicator{iii} \textbf{Exclusion Restriction}: $Z$ has no direct effect on $Y$ other than through $T$, i.e., $Y(z, t) = Y(t)$ for all $z, t \in \{0, 1\}$. \stepindicator{iv} \textbf{Unconfoundedness of the Instrument}: Conditional on observed covariates, the instrumental variable is independent of potential outcomes, i.e., $Z \indep Y(t) \mid \vect X$ for all $t \in \{0, 1\}$.
\end{assumption}

\xhdr{Balke-Pearl Bounds} In the binary IV regime, \citet{balke1997bounds} derived closed-form bounds on the conditional average treatment effect by expressing them as the maximum and minimum of a set of linear functions of the conditional probabilities $p_{yt.z}(\vect x) \coloneqq \P(Y = y, T = t \mid Z = z, \vect X = \vect x)$ for $y, t, z \in \{0, 1\}$. We denote these conditional bounds by $[\lowerbound(\vect x), \upperbound(\vect x)]$. The full equations are provided in \cref{appx:balke-pearl-eqs}. Note that these closed-form bounds do not appear in the training/inference pipeline of \method. We use them only for evaluation purposes on synthetic benchmarks.
\section{Bounding SATE via Bayesian Inference}
\label{sec:bayesian}
Consider a parameterized family of joint distributions $\{P_\psi\}_{\psi \in \Psi}$ over observed variables and all the potential treatments and outcomes. We specifically consider the following factorization:
\begin{align}
\label{eq:joint_distribution}
    &P_\psi \big( \vect X,\ Z,\ \{T(z)\}_{z \in \{0,1\}},\ \{Y(t)\}_{t \in \{0,1\}},\ T,\ Y \big) \\
    &= P_\psi \big( \vect X  \big) \cdot P_\psi \lpar Z \mid \vect X \rpar \cdot P_\psi \big( \{T(z)\}_{z \in \{0,1\}}, \{Y(t)\}_{t \in \{0,1\}} \mid \vect X \big)  \cdot \bbI \big\{ T = T(Z), Y = Y(T) \big\}. \nonumber
\end{align}
This factorization is specifically designed to satisfy \cref{assump:iv}. The term {\small $\bbI \big\{ T \!=\! T(Z), Y \!=\! Y(T) \big\}$} enforces consistency. $Z$ is conditionally independent of potential outcomes given $\vect X$, and the potential outcome only depends on the intervention $T \coloneq t$. 
 
Now, consider a prior $\pi$ on $\Psi$. Given an observational dataset $\data$, we can compute a posterior belief on the parameters using Bayes' rule, $\pi \lpar \psi \mid \data \rpar \propto \pi(\psi) \times P_\psi^\obs\lpar \data\rpar$, 
where $P_\psi^{\obs}$ is the marginal distribution of $P_\psi$ on the observed random variables.
 
\xhdr{Marginal Posterior of the SATE} For each parameter $\psi$ and dataset $\data$, the joint distribution $P_\psi$ in \cref{eq:joint_distribution} pins down a SATE, i.e., $\lambda(\psi; \data) \coloneqq \frac{1}{|\data|} \sum_{\vect x_i \in \data} \E_{P_\psi}[Y(t = 1) - Y(t = 0) \mid \vect X = \vect x_i]$. We then define the marginal posterior of SATE as
\begin{equation}
\label{eq:ppd-ate}
    \pi_{\text{sate}}\lpar B \mid \data\rpar \coloneqq \int \bbI \lcbar \lambda(\psi; \data) \in B \rcbar \ \pi\lpar \diff \psi \mid \data \rpar,
\end{equation}
where $B$ is an arbitrary Borel subset of $\R$. Under a sufficiently broad prior, $\pi_{\text{sate}}$ assigns non-negligible mass across the entire identified set, so its quantiles serve as natural estimators of the lower and upper bounds. Specifically, for a chosen $\alpha \in (0, 1)$, we define
\begin{equation}
\label{eq:low-up-interval}
    \underline{\lambda}_\alpha(\data)
    \coloneqq
    Q_{\alpha/2}\!\big(\pi_{\text{sate}}\lpar \cdot \mid \data\rpar\big), \; 
    \overline{\lambda}_\alpha(\data)
    \coloneqq
    Q_{1-\alpha/2}\!\big(\pi_{\text{sate}}\lpar \cdot \mid \data\rpar\big),
\end{equation}
where $Q_{\beta}(\mu)$ denotes the $\beta$-quantile of a distribution $\mu$. We return the interval $\lbar \underline{\lambda}_\alpha(\data),\; \overline{\lambda}_\alpha(\data)\rbar$ as our final estimator of the SATE bound. This formulation does \emph{not} require a closed-form estimand for the lower or upper bound, resolving \textbf{L1}.
 
\xhdr{Why Standard Bayesian Inference Fails} Classical Bayesian inference of $\pi_{\text{sate}}$ via MCMC or variational inference  can suffer from three issues along the unidentifiable parts of $\Psi$, where the likelihood $P_\psi^\obs$ is flat~\citep{silva2016causal}. First, MCMC mixes poorly along these directions, and variational inference---which minimizes the mode-seeking exclusive KL---under-disperses, collapsing onto a single mode of the identified set (\textbf{L3}). Second, the posterior shape along these directions is dictated entirely by the prior even with infinite data, producing prior sensitivity (\textbf{L4})~\citep{silva2016causal,jesson2020identifying}. Third, inference must be re-run per dataset, which is computationally prohibitive at scale (\textbf{L5}). \Cref{sec:implementing} addresses \textbf{L3}--\textbf{L5} via amortized in-context learning with an inclusive KL objective.
\section{Implementing the Bayesian Approach via Amortized In-Context Learning}
\label{sec:implementing}
To implement the Bayesian framework described above, we must specify: \stepindicator{i} a computational method for approximating the marginal posterior $\pi_{\text{sate}}$, and \stepindicator{ii} the parameter space $\Psi$ and prior $\pi$. We address the former first via amortized inference, then describe our prior construction in \cref{sec:prior-generation}.
 
\xhdr{Approximating the Posterior via Amortized In-Context Learning}
Rather than computing the posterior for each dataset, we employ amortized inference and train a single model that directly approximates the marginal posterior of the SATE for any input dataset. Specifically, we seek a model $q_\theta\lpar \cdot \mid \data \rpar \approx \pi_{\text{sate}} \lpar \cdot \mid \data \rpar$, where $\theta$ are the model parameters. We learn $\theta$ such that this approximation holds across diverse datasets. Specifically, we optimize the following loss
\begin{equation}
\label{eq:loss-fnc}
    \cL(\theta) \coloneqq \E_{\psi \sim \pi,\, \tilde{\data} \sim P_\psi^{\obs}} \lbar -\log q_\theta\lpar \lambda(\psi; \tilde{\data}) \mid \tilde{\data}\rpar \rbar.
\end{equation}
At each evaluation, we first sample a parameter $\psi \sim \pi$, then sample an observational dataset $\tilde{\data}$ of i.i.d. samples from $P_\psi^{\obs}$. Since $\psi$ is known, the ground-truth SATE $\lambda(\psi; \tilde{\data})$ is computable directly from the data-generating process without any closed-form bound estimand. The loss is the negative log-likelihood under $q_\theta$ of the ground-truth SATE values for datasets sampled from the prior.
 
\xhdr{Equivalence to Inclusive KL}
A key property of the loss in \cref{eq:loss-fnc} is that its optimization is equivalent to minimizing the expected inclusive KL between the true marginal posterior and our model:
\begin{equation}
\label{eq:forward-kl-equivalence}
    \arg\min_\theta \cL(\theta) \;=\; \arg\min_\theta \, \E \lbar \kl{\pi_{\text{sate}}\lpar \cdot \mid \tilde{\data}\rpar}{q_\theta\lpar \cdot \mid \tilde{\data}\rpar}\rbar,
\end{equation}
where the outer expectation is over datasets drawn from the prior. Similar derivations have been shown in prior work\citep{müller2023transformers,nagler2023statistical}; for completeness, we provide ours in \cref{appx:forward-kl}.
 
This equivalence is the key technical reason \method\ resolves \textbf{L3}. The inclusive KL is \emph{mass-covering}: minimizing it forces $q_\theta$ to assign positive density wherever $\pi_{\text{sate}}$ has mass, which prevents the under-dispersed posteriors that arise when minimizing the mode-seeking exclusive KL via standard variational inference. In partial identification, where $\pi_{\text{sate}}$ is intentionally diffuse over the identified set, this property is essential for valid bound estimation.

\xhdr{Architecture and Inference}
We follow a similar architecture to standard PFNs \citep{hollmann2023tabpfn,balazadeh2025causalpfn} and model $q_\theta$ as a transformer encoder taking the entire observational dataset $\tilde{\data}$ as its context. We embed each row $\vect o^{(i)} \in \tilde{\data}$ into a single token with a fixed dimension via a learnable embedding layer. The encoder processes the resulting sequence of $|\tilde{\data}|$ tokens through a stack of self-attention and MLP layers. To produce a dataset-level representation, we average the output tokens of the encoder. We omit positional encoding to enforce permutation-invariance over the rows. Full training and architectural details, including hyperparameters, appear in \cref{appx:training-infer-ate}.
 
\subsection{Prior Data Generation}
\label{sec:prior-generation}
Optimizing the training loss in \cref{eq:loss-fnc} does \emph{not} require an explicit parametrization for $P_\psi$. Instead, it only requires a prior that can: \stepindicator{i} generate observational datasets $\tilde{\data}$ consisting of tuples $\lpar \vect x, z, t, y\rpar$ from $P_\psi^\obs$ while adhering to the IV structure in \cref{eq:joint_distribution}, and \stepindicator{ii} compute the ground-truth SATE $\lambda(\psi; \tilde{\data})$. Still, not every choice of prior leads to useful posteriors. If the prior is not sufficiently rich and fails to capture a diverse range of DGPs, then even the optimal approximation $q_{\theta^\star} \approx \pi_{\text{sate}}$ will not generalize to new unseen datasets. Moreover, the prior must distribute mass broadly over the range of possible SATE values: if it concentrates near zero, the trained model will inherit that bias and produce overly narrow posteriors at inference time (\textbf{L4}).
 
To address diversity, we draw inspiration from the success of tabular foundation models~\citep{hollmann2025accurate,ma2025tabdpt,qu2025tabicl} and construct a prior that contains a rich set of datasets with sufficient coverage to approximate real-world scenarios. Our setting, however, introduces a constraint absent in standard supervised tabular models: the joint distribution $P_\psi$ must satisfy the IV constraints stated in \cref{assump:iv}, i.e., each $P_\psi$ must admit the factorization in \cref{eq:joint_distribution}, with the instrument $Z$ being unconfounded given covariates and having no direct effect on the outcome beyond its influence through the treatment. Hence, we propose the following prior data generation process, mirroring the factorization in \cref{eq:joint_distribution}:
 
1. \textbf{Base table and covariate generation} $P_\psi \big( \vect X  \big)$: We begin by generating a ``base'' table with varying numbers of rows and columns, drawn from a large collection of synthetic datasets. We follow the prior data generation scheme used by tabular foundation models in \citet{hollmann2025accurate,balazadeh2025causalpfn}: features are produced by passing samples through randomly initialized neural networks, ensuring the prior covers a rich set of tabular feature distributions. From the base table, we randomly select columns and denote them as our covariates $\vect X$.

2. \textbf{Instrument propensity} $P_\psi \lpar Z \mid \vect X \rpar$: To guarantee the unconfoundedness of the instrument (\cref{assump:iv}), we generate the instrument $Z$ as \emph{only} a function of the covariates. We create a random function $f_\psi: \R^{d'} \to \R$, using synthetic function generators similar to \citet{balazadeh2025causalpfn}. The instrument propensity is then defined as $P_\psi\lpar Z = 1 \mid \vect X \rpar = \sigmoid\lpar f_\psi(\vect X)\rpar$.

3. \textbf{Potential outcome and treatment generation} $P_\psi \big( \{T(z)\}_{z \in \{0,1\}},\ \{Y(t)\}_{t \in \{0,1\}} \mid \vect X \big)$: The conditional distribution of the potential treatments and outcomes can be characterized by a $16$-dimensional probability vector, where each of the $16$ entries corresponds to one of the $|\{0, 1\}^4| = 16$ configurations of the binary $T(z=0), T(z=1), Y(t=0)$, and $Y(t=1)$. We denote this probability vector by $\vect p_\psi(\vect X) \in \Delta^{15}$ in the probability simplex. 

To sample these conditional probability vectors, we select $16$ additional columns from the base table to get a raw ${\vect g^0_\psi}(\vect X) \in \R^{16}$ and apply a softmax function to obtain an initial $\vect p^0_\psi(\vect X) = \softargmax({\vect g^0_\psi}(\vect X))$. 
To ensure that the resulting samples cover the entire probability simplex (\textbf{L4}), we apply a post-processing step. We first sample a target average vector $\overline{\vect p}_\psi \sim \mathrm{Dirichlet}(\gamma, \ldots, \gamma)$ in the simplex. We choose a small concentration parameter (e.g., $\gamma = 0.1$) to get draws that are sparse and spread broadly over the simplex. We then add a single constant vector $\vect c_\psi \in \R^{16}$ to the pre-softmax logits of every individual: ${\vect g_\psi}(\vect x_i) \leftarrow {\vect g^0_\psi}(\vect x_i) + \vect c_\psi$.  We specifically choose $\vect c_\psi = \log \overline{\vect p}_\psi - \log (\sum_{i} \vect p_\psi(\vect x_i)/N)$, 
where the logarithms are componentwise. We then calculate the final probability vector by applying $\vect p_\psi(\vect X) = \softargmax({\vect g_\psi}(\vect X))$ again. This log-domain shift translates each per-individual logit toward the target by an amount equal to the log-discrepancy between target and current population mean. The shift is approximate (because $\softargmax$ does not commute with averaging) but is sufficient to push the population mean substantially toward $\overline{\vect p}_\psi$, and consequently to reshape the prior distribution of $\lambda(\psi; \tilde \data)$ to span a wide range of values across $[-1, 1]$. \cref{fig:prior-support} illustrates the effect on the prior distribution of $\lambda(\psi; \tilde \data)$.
 
4. \textbf{Computing the SATE and simulation of observational data}: With the probability vector $\vect p_\psi$ at hand, we can analytically compute
$\lambda(\psi; \tilde \data) = \frac{1}{N} \sum_{i=1}^N \E_{\vect p_\psi}[Y(t=1) - Y(t=0) \mid \vect X = \vect x_i]$. Moreover, to generate the observational data $\tilde{\data}$, we sample $T(z=0), T(z=1), Y(t=0)$, and $Y(t=1)$ from $\vect p_\psi (\vect X)$ and the instrument $Z$ from the $P_\psi(Z\mid \vect X)$. We then construct the observed treatment and outcome via consistency (\cref{assump:iv}) as $T = T(z=Z)$ and $Y = Y(t=T)$.
\begin{figure}[t]
    \centering
    \captionsetup{font=footnotesize}
    \includegraphics[width=0.86\linewidth]{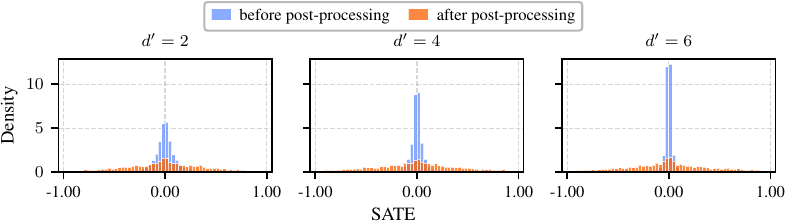}
    \caption{\textbf{The effect of post-processing on SATE distribution}. $d'$ refers to the covariate dimension.}
    \label{fig:prior-support}
    \vspace{-2em}
\end{figure}
 
\xhdr{Remark} When the outcome is continuous, we apply the discretization method in \citet[\S6]{levis2025}; this reduces a continuous outcome IV into solving multiple binary outcome IVs. Because computing the SATEs requires only the synthetic DGP and not closed-form bounds, the framework does not, in principle, require the binary-outcome assumption.
\section{A More Realistic Evaluation: Converting RCTs to IV Benchmarks}
\label{sec:rct_to_iv_benchmark}
Evaluating IV methods is challenging because ground-truth causal effects are rarely known for real datasets. Fully synthetic benchmarks allow exact evaluation but may not reflect realistic data distributions. To bridge this gap, we propose a procedure to convert RCTs into observational IV datasets while preserving the ground-truth causal effects from the experiment.
 
RCTs provide gold-standard estimates of causal effects through randomization. However, RCT data cannot directly serve as IV benchmarks because there is no instrument-treatment-confounder structure. Our key insight is that we can \emph{synthetically} introduce such structure while using accept-reject sampling to ensure the resulting observational outcomes match the original RCT outcomes. This preserves the ground-truth causal effect while creating a realistic IV estimation problem.
 
Given an RCT dataset of i.i.d. samples from $(\vect X, T^{\text{rct}}, Y)$ with \emph{balanced} treatment assignment, we construct an IV dataset as follows. \stepindicator{i} Split $\vect X$ into observed covariates $\vect O$ (available to estimators) and hidden confounders $\vect U$. \stepindicator{ii} Sample the instrument $Z \sim \bernoulli(p_z(\vect O))$, where the propensity $p_z$ depends only on observed covariates. This ensures the instrument satisfies conditional unconfoundedness. \stepindicator{iii} Sample the treatment $T^{\text{synth}} \sim \bernoulli(p_t(\vect O, \vect U, Z))$, where $p_t$ depends on observed covariates, hidden confounders, and the instrument. Note that $p_t$ is synthetically designed to introduce confounding through dependence on $\vect U$; the dependence on $Z$ provides instrument relevance. \stepindicator{iv} Retain samples if and only if $T^{\text{synth}} = T^{\text{rct}}$. For accepted units, use the original $Y$. The following establishes that this procedure produces valid IV data, proved in \cref{appx:proof-rct-to-iv}.
 
\begin{proposition}
\label{prop:rct-to-iv}
Suppose the original RCT has a treatment assignment with $\P(T^{\text{rct}} = 1) = 0.5$, $p_t(\vect O, \vect U, Z) \in [\epsilon, 1-\epsilon]$, and $p_z(\vect O) \in [\epsilon', 1-\epsilon']$ for some $\epsilon, \epsilon' > 0$. Then the dataset resulting from the above procedure satisfies the following.
 
\xhdr{(IV Validity)} The data-generating process satisfies \cref{assump:iv}.
    
\xhdr{(Causal Effect Preservation)} The population average treatment effect (PATE) in the accepted population equals the RCT PATE:
    \begin{equation*}
    \begin{aligned}
        &\E[Y(1) \!-\! Y(0) \!\mid\! T^{\text{rct}} = T^{\text{synth}} ] = \E[Y\! \!\mid\! T^{\text{rct}} \!=\! 1] - \E[Y\! \!\mid\! T^{\text{rct}} \!=\! 0].
    \end{aligned}
    \end{equation*}
\end{proposition}

Note that while our target estimand is the SATE, \cref{prop:rct-to-iv} establishes preservation of the PATE. However, the two are connected: the SATE is itself a Monte Carlo estimator of the PATE over the covariate distribution, and the RCT PATE (our ground-truth label) is estimated via its own Monte Carlo average over the available RCT samples. Both sources of approximation vanish as sample size grows and remain practically reliable for the RCTs we use. 

\xhdr{Why this is a useful benchmark}
Our procedure decouples the components of an IV benchmark by their realism. In particular, the entire outcome distribution is preserved exactly: accepted units retain their original RCT outcomes, so we evaluate against the real, complex outcome-generating mechanism (e.g., job-training effects on earnings) rather than a simulated one. What we synthesize is only the treatment/instrument assignment structures, which are typically simpler than the outcome process. The procedure should be viewed not as a replacement for real IV data, but as a controlled bridge between toy simulators and real IV settings without ground-truth.

\section{Experiments}
\label{sec:experiments}
\setlength{\textfloatsep}{6pt plus 2pt minus 2pt}
\setlength{\intextsep}{6pt plus 2pt minus 2pt}
\setlength{\abovecaptionskip}{3pt}
\setlength{\belowcaptionskip}{0pt}
\xhdr{Baselines} We choose the following standard baselines for IV.

\emph{(i) Plug-in Estimator}: The plug-in estimator from \cref{sec:background}, with TabPFN-v2 as the base classifier; requires the Balke-Pearl equations.
 
\emph{(ii) Nonparametric Efficient Estimators}: The efficient frequentist methods of \citet{levis2025,whitehouse2025inference}, which requires the Balke-Pearl bounds; we use TabPFN-v2 as the base model. 

\emph{(iii) Bayes-BP (Bayesian Inference over the Bounds)}: Bayesian inference on the conditional probabilities $p_{yt.z}(\vect x)$ under a GP prior; we sample posteriors, evaluate the Balke-Pearl equations on each sample, and average. This baseline isolates epistemic uncertainty in the identifiable $p(Y, T \mid Z, \vect X)$~\citep{richardson2011transparent,silva2016causal}, but still requires the closed-form bounds and per-dataset optimization.
 
\emph{(iv) Bayes (Pure Bayesian Inference)}: The fully Bayesian approach discussed in \cref{sec:bayesian} that directly infers the posterior on the SATE without using the Balke-Pearl equations. We follow the factorization in \cref{eq:joint_distribution} and use a GP prior. This baseline mirrors the formulation \method\ amortizes with a different prior, allowing us to isolate the effect of replacing MCMC and variational inference with amortized in-context learning, and the GP prior with the one discussed in \cref{sec:prior-generation}. For both Bayesian baselines we use sparse variational inference for scalability~\citep{titsias2009variational,gardner2018gpytorch}.

\xhdr{Quantile levels} The Bayes baseline and \method\ both return a marginal posterior over the SATE; we convert this posterior into a bound interval by taking the $\alpha/2$ and $1-\alpha/2$ quantiles as in \cref{eq:low-up-interval}. We use $\alpha=0.01$ for both. For Bayes-BP, which targets the lower/upper bounds directly, we use $\alpha = 0.1$.

\xhdr{Evaluation Metrics} We evaluate methods along three axes: \emph{validity}, \emph{informativeness}, and \emph{efficiency}. All the metrics are averaged over 10 random seeds, and we report mean $\pm$ standard error.
 
\emph{(i) Validity}: When the true bounds $[\lowerbound^\star, \upperbound^\star]$ are known---only in the fully synthetic benchmark where the complete DGP is available---we report a binary indicator that equals 1 if the estimated interval contains the true interval, i.e., $\hat{\lowerbound} \le \lowerbound^\star$ and $\hat{\upperbound} \ge \upperbound^\star$. For RCT-derived benchmarks, the sharp IV bounds are not known; the RCT provides an estimated causal-effect label, but not the induced IV identified set. We therefore report a binary indicator that equals 1 if $\lambda^\star \in [\hat{\lowerbound}, \hat{\upperbound}]$. The distinction will be clear based on the benchmark. \emph{(ii) Informativeness (Width)}: We report the normalized width of the estimated interval divided by the width of the Manski natural bounds~\citep{manski1990nonparametric} as the measure of informativeness/tightness. Note that the width of the Manski bounds is always $w = (\max(Y) - \min(Y))$. Hence, our informative measure is defined as $(\hat{\upperbound} - \hat{\lowerbound})/w$. A tighter estimated interval ($\ll 1$) indicates the method extracts more information from the IV structure.  \emph{(iii) Efficiency}:  We record the inference time in seconds per $1{,}000$ samples of the inference dataset to measure efficiency.
 
Note that an interval estimator is only useful if it is valid, i.e., it contains the true causal quantity. Informativeness and efficiency metrics are mostly meaningful only after validity is achieved; a narrow interval that excludes the truth can lead to systematically incorrect decisions.
 
\xhdr{Datasets}
We evaluate \method\ on fully synthetic and semi-synthetic benchmarks. Below, we discuss their high-level construction. For details like the characterization of the underlying DGPs, see \cref{appx:benchmarks}. We also consider a modified version of the airplane demand dataset~\citep{hartford2017deep} in \cref{appx:airplane}. 
 
\emph{Binary Outcome Benchmark (Synthetic)}: We construct a synthetic IV benchmark with binary instrument, treatment, and outcome. Covariates $\vect X \in \R^d$ are sampled with $d \sim \text{Unif}(5, 10)$. The joint distribution of potential treatments and outcomes is parameterized via principal strata~\citep{frangakis2002principal}, with strata probabilities generated as softmax transformations of linear functions of the covariates. The instrument $Z$ is generated as a function of covariates with added exogenous noise. This construction allows us to compute the exact Balke-Pearl bounds, which we use as ground truth for evaluation.
 
\emph{Jobs (Semi-Synthetic)}: 
The National Supported Work (NSW) Demonstration is a randomized job-training program designed to evaluate the causal effect of training on subsequent earnings~\citep{lalonde1986evaluating}. We apply the RCT-to-IV procedure in \cref{sec:rct_to_iv_benchmark} to convert the dataset into a suitable IV benchmark. We use log-transformed 1978 earnings as the outcome. Observed covariates include demographics (age, education, race, marital status), while pre-program earnings (1974, 1975) serve as hidden confounders. We also vary the instrument strength, controlling the correlation between the instrument and treatment $\rho(Z, T)$, creating ``weak'' ($\rho$ small) and ``strong'' ($\rho$ large) instrument variants.
 
\emph{STAR (Semi-Synthetic)}: 
Project STAR is a large-scale RCT conducted to measure the effect of class size on early-grade academic achievement~\citep{mosteller1995tennessee,aerpackage}. We apply the RCT-to-IV procedure discussed in \cref{sec:rct_to_iv_benchmark} and construct an entry-grade snapshot with grade-standardized test scores (math or reading) as outcomes. Observed covariates include student demographics and teacher credentials, while school-level factors (free lunch eligibility, school type) serve as hidden confounders. We consider two treatment contrasts: \emph{small vs.\ regular} and \emph{regular+aide vs.\ regular}, each with weak and strong instrument variants. Here, we present the results for the small vs. regular class comparison with reading scores as outcome. For other variants, see \cref{appx:star-full}.

For RCT-derived benchmarks, we balance the treatment assignment by randomly subsampling the larger arm to match the smaller one in size. This satisfies the required condition of \cref{prop:rct-to-iv}.
 
\subsection{Results}
\label{sec:results}
 
\begin{table}[t]
\centering
\setlength{\tabcolsep}{4pt}\renewcommand{\arraystretch}{0.95} 
\captionsetup{font=footnotesize}
\caption{Results on the synthetic binary-outcome benchmark (\emph{left}; ground-truth bounds known) and the semi-synthetic Jobs benchmark (\emph{right}; RCT-derived labels available, aggregated across weak and strong instruments). Best validity and time in \textbf{bold}; best width among perfectly valid methods is \underline{underlined}. Mean $\pm$ ste over seeds.}
\label{tab:binary-and-jobs}
\begin{adjustbox}{max width=0.85\textwidth}
\begin{tabular}{lccc|ccc}
\toprule
& \multicolumn{3}{c|}{\textbf{Binary-outcome (synthetic)}} & \multicolumn{3}{c}{\textbf{Jobs (semi-synthetic)}} \\
\cmidrule(lr){2-4}\cmidrule(lr){5-7}
\textbf{Method} & \textbf{Validity}$\uparrow$ & \textbf{Width}$\downarrow$ & \textbf{Time/1k (s)}$\downarrow$ & \textbf{Validity}$\uparrow$ & \textbf{Width}$\downarrow$ & \textbf{Time/1k (s)}$\downarrow$ \\
\midrule
\method\ (Ours) & \textbf{1.00 {\footnotesize $\pm$ 0.00}} & \underline{0.609 {\footnotesize $\pm$ 0.158}} & \textbf{0.18 {\footnotesize $\pm$ 0.01}} & \textbf{1.00 {\footnotesize $\pm$ 0.00}} & \underline{0.439 {\footnotesize $\pm$ 0.040}} & \textbf{1.77 {\footnotesize $\pm$ 0.29}} \\
\midrule
Plug-in & 0.50 {\footnotesize $\pm$ 0.17} & 0.530 {\footnotesize $\pm$ 0.147} & 1.18 {\footnotesize $\pm$ 0.26} & 0.70 {\footnotesize $\pm$ 0.09} & 0.541 {\footnotesize $\pm$ 0.029} & 68.5 {\footnotesize $\pm$ 17.0} \\
Whitehouse et al. & \textbf{1.00 {\footnotesize $\pm$ 0.00}} & 0.641 {\footnotesize $\pm$ 0.158} & 2.15 {\footnotesize $\pm$ 0.35} & 0.90 {\footnotesize $\pm$ 0.07} & 0.836 {\footnotesize $\pm$ 0.040} & 92.9 {\footnotesize $\pm$ 10.3} \\
Levis et al. & 0.90 {\footnotesize $\pm$ 0.10} & 0.575 {\footnotesize $\pm$ 0.147} & 2.30 {\footnotesize $\pm$ 0.35} & 0.60 {\footnotesize $\pm$ 0.07} & 0.545 {\footnotesize $\pm$ 0.053} & 92.6 {\footnotesize $\pm$ 16.0} \\
Bayes-BP & 0.70 {\footnotesize $\pm$ 0.15} & 0.588 {\footnotesize $\pm$ 0.113} & 6.35 {\footnotesize $\pm$ 0.32} & 0.90 {\footnotesize $\pm$ 0.10} & 0.429 {\footnotesize $\pm$ 0.089} & 136.46 {\footnotesize $\pm$ 16.09} \\
Bayes & 0.00 {\footnotesize $\pm$ 0.00} & 0.169 {\footnotesize $\pm$ 0.038} & 8.82 {\footnotesize $\pm$ 0.35} & 0.25 {\footnotesize $\pm$ 0.14} & 0.229 {\footnotesize $\pm$ 0.045} & 254.55 {\footnotesize $\pm$ 19.93} \\
\bottomrule
\end{tabular}%
\end{adjustbox}
\end{table}
 
\begin{table}[t]
\centering
\setlength{\tabcolsep}{4pt}\renewcommand{\arraystretch}{0.95} 
\captionsetup{font=footnotesize}
\caption{Results on STAR (small vs.\ regular class size, reading scores), stratified by instrument strength.}
\label{tab:star-small-regular-reading}
\begin{adjustbox}{max width=0.85\textwidth}
\begin{tabular}{lccc|ccc}
\toprule
& \multicolumn{3}{c|}{\textbf{Weak instrument} $\rho(Z, T) \approx 0.29$} & \multicolumn{3}{c}{\textbf{Strong instrument} $\rho(Z, T) \approx 0.89$} \\
\cmidrule(lr){2-4}\cmidrule(lr){5-7}
\textbf{Method} &
\textbf{Validity}$\uparrow$ & \textbf{Width}$\downarrow$ & \textbf{Time/1k (s)}$\downarrow$ &
\textbf{Validity}$\uparrow$ & \textbf{Width}$\downarrow$ & \textbf{Time/1k (s)}$\downarrow$ \\
\midrule
\method\ (Ours) & \textbf{1.00 {\footnotesize $\pm$ 0.00}} & \underline{0.232 {\footnotesize $\pm$ 0.041}} & \textbf{0.17 {\footnotesize $\pm$ 0.01}} & \textbf{1.00 {\footnotesize $\pm$ 0.00}} & \underline{0.044 {\footnotesize $\pm$ 0.002}} & \textbf{0.18 {\footnotesize $\pm$ 0.03}} \\
\midrule
Plug-in & \textbf{1.00 {\footnotesize $\pm$ 0.00}} & 0.661 {\footnotesize $\pm$ 0.017} & 9.51 {\footnotesize $\pm$ 1.34}
        & \textbf{1.00 {\footnotesize $\pm$ 0.00}} & 0.103 {\footnotesize $\pm$ 0.009} & 8.05 {\footnotesize $\pm$ 0.99} \\
Whitehouse et al. & \textbf{1.00 {\footnotesize $\pm$ 0.00}} & 0.817 {\footnotesize $\pm$ 0.017} & 15.6 {\footnotesize $\pm$ 1.9}
                  & \textbf{1.00 {\footnotesize $\pm$ 0.00}} & 0.149 {\footnotesize $\pm$ 0.013} & 16.5 {\footnotesize $\pm$ 2.7} \\
Levis et al. & \textbf{1.00 {\footnotesize $\pm$ 0.00}} & 0.692 {\footnotesize $\pm$ 0.018} & 17.5 {\footnotesize $\pm$ 2.9}
             & \textbf{1.00 {\footnotesize $\pm$ 0.00}} & 0.110 {\footnotesize $\pm$ 0.011} & 16.0 {\footnotesize $\pm$ 1.9} \\
Bayes-BP & \textbf{1.00 {\footnotesize $\pm$ 0.00}} & 0.711 {\footnotesize $\pm$ 0.015} & 81.72 {\footnotesize $\pm$ 5.72} & \textbf{1.00 {\footnotesize $\pm$ 0.00}} & 0.220 {\footnotesize $\pm$ 0.013} & 84.31 {\footnotesize $\pm$ 5.92} \\
Bayes  & 0.90 {\footnotesize $\pm$ 0.11} & 0.084 {\footnotesize $\pm$ 0.005} & 175.69 {\footnotesize $\pm$ 24.57} & \textbf{1.00 {\footnotesize $\pm$ 0.00}} & 0.148 {\footnotesize $\pm$ 0.005} & 170.35 {\footnotesize $\pm$ 19.03} \\
\bottomrule
\end{tabular}%
\end{adjustbox}
\vspace{-0.1cm}
\end{table}
 
\xhdr{Synthetic and Jobs benchmarks} \cref{tab:binary-and-jobs} reports validity, width, and per-1k-sample inference time on the synthetic benchmark (where ground-truth bounds are known) and the semi-synthetic Jobs benchmark (where an RCT-derived causal-effect label is available via the RCT-to-IV procedure). On both benchmarks, \method\ is the only method that attains $1.00$ validity \emph{and} the tightest valid intervals: on Jobs, every other baseline either under-covers. The Bayes baseline, the same fully-Bayesian formulation \method\ amortizes but with variational inference, collapses to validity $0.00$ on the synthetic benchmark and $0.25$ on Jobs, with severely under-dispersed widths. This is the predicted consequence of mode-seeking exclusive KL discussed in \cref{sec:implementing} and confirms the importance of the forward-KL training objective. \method\ is also the fastest method by $2$--$140\times$ across both benchmarks.
 
\xhdr{STAR benchmark} \cref{tab:star-small-regular-reading} reports results on STAR (small vs.\ regular, reading scores). Most methods attain perfect validity in both weak- and strong-instrument regimes, so the comparison reduces to width. \method\ produces the tightest valid intervals---$2$--$3\times$ tighter than the second-best method---and is the fastest, by $20$--$500\times$. Additional STAR variants (regular+aide vs.\ regular; math scores) and the airplane-demand benchmark appear in \cref{appx:star-full,appx:airplane}.
\begin{wrapfigure}{r}{0.45\linewidth}
\centering
\captionsetup{font=footnotesize}
\includegraphics[width=0.95\linewidth]{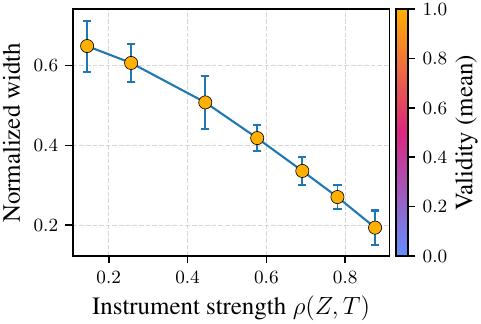}
\caption{\textbf{Effect of instrument strength on interval width} (Jobs benchmark). Width decreases as $\rho(Z,T)$ increases; validity remains perfect.}
\label{fig:instrument-sweep}
\vspace{-1em}
\end{wrapfigure}
 
\xhdr{Effect of instrument strength} \cref{fig:instrument-sweep} shows that as the instrument strength $\rho(Z, T)$ increases on the Jobs benchmark, \method's interval narrows monotonically while validity remains at $1.00$ across the full range. This is a sanity check to see that a stronger instrument indeed shrinks the identified set, and \method's mass-covering posterior contracts accordingly.

\xhdr{Calibration and sensitivity analysis} Beyond the validity metric reported above, \cref{appx:calibration-sensitivity} provides a fuller picture of \method's interval estimator. Calibration curves (\cref{appx:calibration}) show that empirical coverage exceeds the nominal level across all credibility levels on three synthetic DGPs, which is consistent with the conservative over-coverage induced by the inclusive KL objective. A sensitivity analysis (\cref{appx:sensitivity}) further shows that coverage always remains perfect at 1.00, while interval width contracts monotonically with sample size and inflates modestly with covariate dimensionality. 
\section{Conclusion and Limitations}
\label{sec:conclusion}
We introduced \method, a Bayesian in-context learning approach for partial identification under instrumental variables. By learning the marginal posterior of the sample average treatment effect (SATE) directly and minimizing an inclusive KL training objective, \method\ addresses five key limitations of partial identification under IV (\textbf{L1}--\textbf{L5}). Our experiments demonstrate that \method\ achieves superior validity--informativeness trade-offs compared to existing baselines with orders-of-magnitude faster inference. Beyond the methodological contribution, we introduced a prior data generation procedure that produces a diverse family of IV-consistent synthetic datasets with full support over the SATE that enabled large-scale pre-training. We also proposed a novel procedure to convert randomized controlled trial data into realistic IV benchmarks with preserved RCT causal-effect labels. This is useful for evaluating partial identification methods on realistic data.

Several limitations remain. First, while \method\ achieves perfect validity across our benchmarks, valid intervals are not formally guaranteed under arbitrary prior--data mismatch. 
The inclusive KL objective enforces mass-covering only with respect to the prior used at training time; downstream coverage on real data depends on the prior assigning sufficient mass near the true DGP. We empirically probed this dependence through calibration curves and analyzing \method's sensitivity to sample-size and covariate size, but achieving exact frequentist calibration in amortized methods remains an open problem~\citep{mourao2026prior-data,melnychuk2026frequentist}. Second, although \method\ is architecturally agnostic to outcome type, the current implementation samples binary outcomes during pretraining and discretizes continuous outcomes at inference time. We leave the implementation of the fully continuous setting to future work.

\bibliographystyle{plainnat}
\bibliography{main}

\clearpage
\appendix

\section*{Appendix Contents}
\addcontentsline{toc}{section}{Appendix Contents}
\startcontents[sections]
\printcontents[sections]{}{1}{\setcounter{tocdepth}{2}}
\clearpage

\crefalias{section}{appendix}
\crefalias{subsection}{appendix}
\Crefname{appendix}{Appendix}{Appendices}
\crefname{appendix}{appendix}{appendices}

\section{Balke-Pearl Equations}
\label{appx:balke-pearl-eqs}
 
The closed-form sharp bounds of \citet{balke1994counterfactual,balke1997bounds} for the standard binary IV model arise from a constrained linear program whose decision variables are the probabilities of the sixteen response-function strata for a binary $(T, Y)$ pair under a binary instrument $Z$. The exclusion restriction and consistency (\cref{assump:iv}) impose linear equality constraints between these strata probabilities and the observational distribution $p_{yt.z}(\vect x) = \P(Y = y, T = t \mid Z = z, \vect X = \vect x)$ for $y, t, z \in \{0, 1\}$. Maximizing or minimizing the linear functional $\E[Y(1) - Y(0) \mid \vect X = \vect x]$ subject to these constraints yields, by LP duality, sharp bounds expressible as the maximum of eight linear lower-bound expressions $\phi_{l,j}$ and the minimum of eight linear upper-bound expressions $\phi_{u,j}$, each a different linear combination of the eight conditional probabilities $\{p_{yt.z}(\vect x)\}$. Each $\phi_{l,j}$ corresponds to a worst-case allocation of mass across the unidentifiable response-function strata under one extremum of the LP, with an analogous reading for $\phi_{u,j}$. The bounds are pointwise sharp~\citep{balke1997bounds}: there exist data-generating processes consistent with $\{p_{yt.z}(\vect x)\}$ at which the conditional SATE attains each endpoint. \citet{levis2025} extend sharpness from the conditional SATE to the population SATE. The eight expressions for the lower and upper conditional bounds are:
 
\begin{align}
\phi_{l,1}(\vect x) &= p_{11.1}(\vect x) + p_{00.0}(\vect x) - 1 \nonumber \\
\phi_{l,2}(\vect x) &= p_{11.0}(\vect x) + p_{00.1}(\vect x) - 1 \nonumber \\
\phi_{l,3}(\vect x) &= -p_{01.1}(\vect x) - p_{10.1}(\vect x) \nonumber \\
\phi_{l,4}(\vect x) &= -p_{01.0}(\vect x) - p_{10.0}(\vect x) \nonumber \\
\phi_{l,5}(\vect x) &= p_{11.0}(\vect x) - p_{11.1}(\vect x) - p_{10.1}(\vect x) - p_{01.0}(\vect x) - p_{10.0}(\vect x)  \\
\phi_{l,6}(\vect x) &= p_{11.1}(\vect x) - p_{11.0}(\vect x) - p_{10.0}(\vect x) - p_{01.1}(\vect x) - p_{10.1}(\vect x) \nonumber \\
\phi_{l,7}(\vect x) &= p_{00.1}(\vect x) - p_{01.1}(\vect x) - p_{10.1}(\vect x) - p_{01.0}(\vect x) - p_{00.0}(\vect x) \nonumber \\
\phi_{l,8}(\vect x) &= p_{00.0}(\vect x) - p_{01.0}(\vect x) - p_{10.0}(\vect x) - p_{01.1}(\vect x) - p_{00.1}(\vect x) \nonumber \\[10pt]
\phi_{u,1}(\vect x) &= 1 - p_{01.1}(\vect x) - p_{10.0}(\vect x) \nonumber \\
\phi_{u,2}(\vect x) &= 1 - p_{01.0}(\vect x) - p_{10.1}(\vect x) \nonumber \\
\phi_{u,3}(\vect x) &= p_{11.1}(\vect x) + p_{00.1}(\vect x) \nonumber \\
\phi_{u,4}(\vect x) &= p_{11.0}(\vect x) + p_{00.0}(\vect x) \nonumber \\
\phi_{u,5}(\vect x) &= -p_{01.0}(\vect x) + p_{01.1}(\vect x) + p_{00.1}(\vect x) + p_{11.0}(\vect x) + p_{00.0}(\vect x) \\
\phi_{u,6}(\vect x) &= -p_{01.1}(\vect x) + p_{11.1}(\vect x) + p_{00.1}(\vect x) + p_{01.0}(\vect x) + p_{00.0}(\vect x) \nonumber \\
\phi_{u,7}(\vect x) &= -p_{10.1}(\vect x) + p_{11.1}(\vect x) + p_{00.1}(\vect x) + p_{11.0}(\vect x) + p_{10.0}(\vect x) \nonumber \\
\phi_{u,8}(\vect x) &= -p_{10.0}(\vect x) + p_{11.0}(\vect x) + p_{00.0}(\vect x) + p_{11.1}(\vect x) + p_{10.1}(\vect x) \nonumber
\end{align}
 
In \method, these equations are used \emph{only} to compute ground-truth bounds for synthetic benchmarks; they do not appear in the training/inference pipeline of the main \method\ method, which operates directly on the marginal posterior of the SATE.

\section{Inclusive KL Equivalence}
\label{appx:forward-kl}

We derive the equivalence stated in \cref{eq:forward-kl-equivalence}. Recall the loss
\begin{equation*}
    \cL(\theta) = \E_{\psi \sim \pi,\, \tilde{\data} \sim P_\psi^{\obs}} \lbar -\log q_\theta\lpar \lambda(\psi; \tilde{\data}) \mid \tilde{\data}\rpar \rbar.
\end{equation*}
By the law of total expectation,
\begin{equation*}
    \cL(\theta) = \E_{\tilde{\data} \sim P^{\obs}} \!\E_{\psi \sim \pi(\cdot \mid \tilde{\data})} \lbar -\log q_\theta\lpar \lambda(\psi; \tilde{\data}) \mid \tilde{\data}\rpar \rbar,
\end{equation*}
where we have used Bayes' rule to swap the joint sampling $(\psi, \tilde{\data}) \sim \pi \cdot P_\psi^{\obs}$ for the marginal-conditional decomposition $\tilde{\data} \sim P^{\obs}$ and $\psi \mid \tilde{\data} \sim \pi(\cdot \mid \tilde{\data})$. The inner expectation is
\begin{equation*}
    \E_{\psi \sim \pi(\cdot \mid \tilde{\data})} \lbar -\log q_\theta\lpar \lambda(\psi; \tilde{\data}) \mid \tilde{\data}\rpar \rbar = \E_{b \sim \pi_{\text{sate}}(\cdot \mid \tilde{\data})} \lbar -\log q_\theta\lpar b \mid \tilde{\data}\rpar \rbar,
\end{equation*}
where the second equality follows from the definition of $\pi_{\text{sate}}$ as the pushforward of $\pi(\cdot \mid \tilde{\data})$ through $\psi \mapsto \lambda(\psi; \tilde{\data})$. Using the cross-entropy decomposition,
\begin{equation*}
    \E_{b \sim \pi_{\text{sate}}(\cdot \mid \tilde{\data})} \lbar -\log q_\theta\lpar b \mid \tilde{\data}\rpar \rbar = H\lpar \pi_{\text{sate}}(\cdot \mid \tilde{\data}) \rpar + \kl{\pi_{\text{sate}}(\cdot \mid \tilde{\data})}{q_\theta(\cdot \mid \tilde{\data})},
\end{equation*}
where $H(\cdot)$ is the entropy. The entropy term does not depend on $\theta$, so
\begin{equation*}
    \arg\min_\theta \cL(\theta) = \arg\min_\theta \, \E_{\tilde{\data} \sim P^{\obs}} \lbar \kl{\pi_{\text{sate}}(\cdot \mid \tilde{\data})}{q_\theta(\cdot \mid \tilde{\data})}\rbar,
\end{equation*}
which is precisely the expected inclusive KL between the true marginal posterior and our approximation. The inclusive KL is mass-covering in the standard sense: any region with $\pi_{\text{sate}}$-mass for which $q_\theta$ assigns zero density contributes infinite penalty, so the optimum $q_{\theta^\star}$ must support the entire support of $\pi_{\text{sate}}$. \qed

\section{Training Details and Inference}
\label{appx:training-infer-ate} 
\xhdr{Output level} The averaged representation is projected through a linear head producing a $K$-dimensional vector ($K = 1{,}024$) of logits over a discretized output space. Since $Y \in [0, 1]$, the SATE lies in $[-1, 1]$; we partition $[-1, 1]$ into $K$ uniform bins and apply softmax to obtain a probability distribution $q_\theta([k] \mid \tilde{\data})$ over bin indices $k \in \{1, \ldots, K\}$.
 
\xhdr{Inferring Bounds} At inference time, given a new observational dataset $\data$, the trained model $q_{\theta^\star}$ produces an approximation to $\pi_{\text{sate}}(\cdot \mid \data)$ in a single forward pass. We read off the bound as $\lbar Q_{\alpha/2}\lpar q_{\theta^\star}(\cdot \mid \data)\rpar,\; Q_{1-\alpha/2}\lpar q_{\theta^\star}(\cdot \mid \data)\rpar\rbar$. Inference is one forward pass per dataset, resolving \textbf{L5}. 

\xhdr{Training Process} We train $\theta$ by performing gradient descent updates with the AdamW optimizer~\citep{loshchilov2017decoupled} on a Monte-Carlo estimate of the loss in \cref{eq:loss-fnc}. At each iteration, we sample a batch of parameters $\psi_1, \ldots, \psi_B \sim \pi$. For each $\psi_b$, we generate an observational dataset $\tilde{\data}_b \sim P_{\psi_b}^\obs$ and compute the ground-truth SATE $\lambda(\psi_b; \tilde{\data}_b)$ analytically. We discretize the SATE values into uniformly spaced bins on $[-1, 1]$ via the mapping $w \mapsto \lfloor K \times (w + 1)/2 \rfloor$ to obtain $[\hat{\lambda}(\psi_b; \tilde{\data}_b)]$ (clipped to be at most $K - 1$). The estimated loss is
\begin{equation}
\label{eq:loss-fnc-estimate}
    \widehat{\cL}(\theta) \coloneqq - \frac{1}{B}\sum_{b=1}^B \log q_\theta\lpar [\hat{\lambda}(\psi_b; \tilde{\data}_b)] \mid \tilde{\data}_b \rpar.
\end{equation}
We minimize $\widehat{\cL}$ for a sufficient number of steps until the loss converges. We denote the trained parameters by $\theta^\star$.
 
\xhdr{Inferring SATE Bounds} Once trained, the transformer model can take new unseen datasets $\data$ and produce $q_{\theta^\star}(\cdot \mid \data) \approx \pi_{\text{sate}}(\cdot \mid \data)$ in a single forward pass without further training. We then read off the bound interval $\lbar Q_{\alpha/2}\lpar q_{\theta^\star}(\cdot \mid \data)\rpar,\; Q_{1-\alpha/2}\lpar q_{\theta^\star}(\cdot \mid \data)\rpar \rbar$ via \cref{eq:low-up-interval}, computing the quantiles directly on the discretized marginal posterior.
 
\xhdr{Hyperparameters} We list the training hyperparameters in \cref{tab:hyperparams}. The transformer architecture is a standard encoder with self-attention layers and MLPs, omitting positional encoding to enforce permutation-invariance over dataset rows. The total pre-training cost is $69.5$ GPU-hours on one H100. Once trained, inference takes under $2$ seconds per dataset, so the pre-training cost amortizes after roughly $1{,}000$--$3{,}000$ inference calls---a threshold readily reached when \method\ is deployed as a general-purpose tool.
 
\begin{table}[ht]
\centering
\caption{Training hyperparameters for \method.}
\label{tab:hyperparams}
\begin{tabular}{ll}
\toprule
\textbf{Hyperparameter} & \textbf{Value} \\
\midrule
Batch size & 256 \\
Iterations & 262{,}144 \\
Number of synthetic tables & 67{,}108{,}864 \\
Output bins ($K$) & 1{,}024 \\
Optimizer & AdamW \\
{Learning rate} & 1e-4 \\
{Weight decay} & 0.05 \\
{Embedding dimension} & 384 \\
{Transformer depth} & 20 \\
{Number of attention heads} & 6 \\
Total training time (one H100 GPU) & 69.5 hrs \\
\bottomrule
\end{tabular}
\end{table}

\section{Proof of Proposition~\ref{prop:rct-to-iv}}
\label{appx:proof-rct-to-iv}
 
\begin{proof}
We prove each claim in turn.
 
\textbf{Part 1: IV Validity.}
 
\emph{Instrument unconfoundedness:} By construction, $Z$ is generated as $Z_i \sim \text{Bernoulli}(p_z(\vect O_i))$, depending only on observed covariates $\vect O$. Since $Z$ does not depend on $\vect U$, we have
\begin{align}
    Z \indep \vect U \mid \vect O.
\end{align}
Moreover, the RCT outcomes $Y(t)$ are pre-determined and do not depend on the synthetically generated $Z$. Since the original RCT satisfies $T^{\text{rct}} \indep Y(t) \mid \vect X = (\vect O, \vect U)$ by randomization, and $Z$ is generated independently of $Y(t)$ given $\vect O$, we have $Z \indep Y(t) \mid \vect O$.
 
\emph{Exclusion restriction:} The outcome $Y_i$ for each accepted unit is the original RCT outcome, which was generated before and independently of the synthetic instrument $Z$. The instrument $Z$ was only used to generate the synthetic treatment $T^{\text{synth}}$, not the outcome. Therefore, $Z$ has no direct effect on $Y$ other than through its effect on $T$.
 
\textbf{Part 2: SATE Preservation.}
 
Let $\gamma = \P(T^{\text{rct}} = 1) = 0.5$ by the balancing assumption. For any unit with covariates $(\vect O, \vect U, Z)$, define $p := p_t(\vect O, \vect U, Z)$. The acceptance probability is:
{\small
\begin{align}
    \P(T^{\text{synth}} = T^{\text{rct}} \mid \vect O, \vect U, Z)
    &= \P(T^{\text{synth}} = 1, T^{\text{rct}} = 1 \mid \vect O, \vect U, Z) + \P(T^{\text{synth}} = 0, T^{\text{rct}} = 0\mid \vect O, \vect U, Z ) \nonumber \\
    &= p \cdot \gamma + (1-p) \cdot (1-\gamma) =p \cdot 0.5 + (1-p) \cdot 0.5 = 0.5.
\end{align}
}%
Note that this acceptance probability does not depend on $(\vect O, \vect U, Z)$. Therefore, the accepted subsample is a uniform random subsample of the original balanced RCT. Moreover, for the accepted units, the treatment assignment follows:
{\small
\begin{align}
    \P(T = 1 \mid \vect O, \vect U, Z, T^{\text{synth}} = T^{\text{rct}}) &= \frac{\P(T^{\text{synth}} = 1, T^{\text{rct}} = 1 \mid \vect O, \vect U, Z)}{\P(T^{\text{synth}} = T^{\text{rct}} \mid \vect O, \vect U, Z)} = \frac{p \cdot 0.5}{0.5} = p = p_t(\vect O, \vect U, Z).
\end{align}
}%
This confirms that the treatment mechanism in the accepted sample is exactly the intended confounded propensity.
 
Since acceptance is independent of all covariates and potential outcomes, conditioning on acceptance does not change the distribution of potential outcomes:
{\small
\begin{align}
    \E[Y(t) \mid T^{\text{synth}} = T^{\text{rct}}] = \E[Y(t)].
\end{align}
}%
Therefore, the SATE in the accepted sample equals the SATE in the original RCT:
{\small
\begin{align}
    \E[Y(1) - Y(0) \mid T^{\text{synth}} = T^{\text{rct}}] &= \E[Y(1) \mid T^{\text{synth}} = T^{\text{rct}}] - \E[Y(0) \mid T^{\text{synth}} = T^{\text{rct}}] \\
    &= \E[Y(1)] - \E[Y(0)] \\
    &= \E[Y \mid T^{\text{rct}} = 1] - \E[Y \mid T^{\text{rct}} = 0],
\end{align}
}%
where the last equality uses the unconfoundedness of the original RCT.
\end{proof}
 
\section{Details of the Benchmark}
\label{appx:benchmarks}
 
\subsection{Details of the Synthetic Binary Outcome Benchmark}
 
We generate 10 synthetic IV datasets with binary instrument, treatment, and outcome, containing $n=2048$ samples:
 
\begin{enumerate}[leftmargin=*]
    
    \item \textbf{Covariates:} Sample $\vect X \in \R^{n \times d}$ with $d \sim \text{Unif}\{5, 6, 7, 8, 9, 10\}$, where each entry is drawn from either $\mathcal{N}(5, 1)$ or $\text{Unif}(-10, 5)$, chosen randomly.
    
    \item \textbf{Instrument generation:} Generate weights $\vect w_Z \in \R^d$ from either $\mathcal{N}(1, 2)$ or $\text{Unif}(-2, 2)$. Compute logits as $\ell_Z = \vect X \vect w_Z + \varepsilon_Z$, where $\varepsilon_Z$ is noise from either $\mathcal{N}(0, 1)$ or $\text{Laplace}(0, 1)$. Standardize: $\tilde{\ell}_Z = (\ell_Z - \bar{\ell}_Z) / \text{std}(\ell_Z)$. Sample $Z \sim \text{Bernoulli}(\sigma(\tilde{\ell}_Z))$.
    
    \item \textbf{Potential treatment/outcome:} Generate weights $\vect W \in \R^{d \times 16}$ and compute logits $\vect L = \vect X \vect W + \vect E$ where $\vect E \in \R^{n \times 16}$ is noise. Apply row-wise softmax to obtain strata probabilities $\vect P \in \R^{n \times 16}$, where each row sums to 1.
    
    \item \textbf{Treatment and outcome strata:} The 16 columns correspond to combinations of:
    \begin{itemize}
        \item Treatment strata: Always-Takers (AT), Never-Takers (NT), Defiers (DE), Compliers (CO)
        \item Outcome strata: Always-Good (AG), Always-Bad (AB), Effective (EF), Harmful (HF)
    \end{itemize}
    
    \item \textbf{Observable generation:} For each unit $i$, sample the stratum from the categorical distribution defined by $\vect P_i$, then determine $(T_i, Y_i)$ based on $Z_i$ and the sampled stratum.
    
    \item \textbf{Ground-truth bounds:} Compute the observational probabilities $p_{yt.z}(\vect x_i)$ analytically from the strata probabilities, then apply the Balke-Pearl equations to obtain $\lowerbound(\vect x_i)$ and $\upperbound(\vect x_i)$.
\end{enumerate}
 
\subsection{Details of the Jobs Benchmark}
 
The original National Supported Work (NSW) Demonstration is an RCT evaluating job training effects on earnings. It includes the following covariates: \texttt{age}, \texttt{education}, \texttt{black}, \texttt{hispanic}, \texttt{married}, \texttt{nodegree}, \texttt{re74}, \texttt{re75}, where the last two features correspond to the earning in 1974 and 1975. The treatment is a binary indicator of assignment to job training program. Finally, the outcome variable is the amount of earnings in 1978 (re78). We apply log transforms to the outputs to get less skewed outcome distribution: $\texttt{re74} \leftarrow \log(\texttt{re74} + 1)$, $\texttt{re75} \leftarrow \log(\texttt{re75} + 1)$, $Y \leftarrow \log(\texttt{re78} + 1)$.
 
\xhdr{Covariate Split} We split the features into observed covariates ($\vect O$), which consist of \texttt{age}, \texttt{education}, \texttt{black}, \texttt{hispanic}, \texttt{married}, \texttt{nodegree}, and unobserved covariates ($\vect U$), which are \texttt{re74}, \texttt{re75} (pre-program earnings). 
 
\xhdr{Synthetic Instrument propensity} We use $p_z(\vect O) = \sigma(s_z(\vect O) + b_z)$, where $s_z(\vect O)$ includes nonlinear effects of the observed demographics and $b_z$ is calibrated such that $\E[Z] \approx 0.5$. In particular, we make sure positivity by clipping $p_z \in [0.05, 0.95]$. 
 
\xhdr{Synthetic Treatment propensity} We use $p_t(\vect O, \vect U, Z) = \sigma(s_t(\vect O, \vect U) + \beta \cdot Z + b_t)$, where $s_t(\vect O, \vect U)$ includes confounding effects through pre-program earnings, and $\beta$ controls instrument strength. We vary $\beta$ from 0.25 to 8 to produce different instrument strengths.
 
\subsection{Details of the STAR Benchmark}
Project STAR randomized students to three class types, \stepindicator{i} Small class ($\approx$ 13--17 students), \stepindicator{ii} Regular class ($\approx$ 22--25 students), and \stepindicator{iii} Regular class with aide ($\approx$ 22--25 students + full-time aide). We create an entry-grade snapshot using the earliest grade with observed class type to convert the longitudinal study into a static dataset.
 
\xhdr{Covariate Split} The observed covariates ($\vect O$) include student's \texttt{gender}, \texttt{birth-quarter}, and \texttt{ethnicity}, as well as \texttt{entry-grade-indicator}. We also consider teacher's \texttt{degree} and \texttt{career-ladder} as observed data. For unobserved variables, we use \texttt{school-type}, \texttt{free-lunch}, and teacher's \texttt{ethnicity} and \texttt{experience}. 
 
\xhdr{Treatment Contrasts} We consider two types of experiments, one with small vs.\ regular class sizes, where $T = 1$ if small class, $T = 0$ if regular class. The other is regular+aide vs. regular classes, with $T = 1$ if regular+aide, $T = 0$ if regular.
 
\xhdr{Outcome Standardization} Finally, we standardize the outcomes, reading and math scores, within the entry grade, i.e., $Y \leftarrow \frac{Y^{\text{raw}} - \mu_G}{\sigma_G}$, where $\mu_G$ and $\sigma_G$ are grade-specific means and standard deviations.
 
\section{Additional Experiments - Airplane Demand Benchmark}
\label{appx:airplane}
We adapt the airplane ticket demand benchmark from \citet{hartford2017deep} to our binary IV setting. The original benchmark simulates airline pricing with unobserved confounding between price and demand. We modify this to have: \stepindicator{i} a binary instrument $Z \sim \text{Bernoulli}(0.5)$, representing high vs.\ low fuel cost regimes, \stepindicator{ii} a binary treatment $T \in \{0, 1\}$, representing regular vs.\ premium pricing, obtained via soft thresholding of the latent continuous price, and \stepindicator{iii} a continuous outcome $Y$ representing ticket sales. The ground-truth SATE is computed analytically from the structural model.
 
The original benchmark models airline ticket pricing and sales as follows. The covariates are $\vect X = (s, \vect e)$ where $s \sim \text{Unif}(0, 10)$ is time-of-year and $\vect e \in \{0,1\}^7$ is a one-hot customer type. The instrument is chosen as $Z \sim \mathcal{N}(0, 1)$ represents fuel cost variation. The treatment value corresponds to the price $P = \frac{\text{sensf}(s)(Z + 3) + 25 + V - \mu_p}{\sigma_p}$ where $V \sim \mathcal{N}(0, \sigma_p^2)$ and $\text{sensf}(\cdot)$ is a seasonal function. Finally, the outcome is defined as $Y = g(s, \vect e, P) + \varepsilon$ where $\varepsilon = \rho \frac{\sigma_y}{\sigma_p} V + \sqrt{1-\rho^2} \eta$ with $\eta \sim \mathcal{N}(0, \sigma_y^2)$. The correlation between $V$ and $\varepsilon$ induces confounding. 
 
We modify the benchmark to have binary instrument and treatment. Specifically, we choose the instrument as $Z \sim \text{Bernoulli}(0.5)$, and map it to standardized proxy using $\tilde{Z} = -(Z - 0.5) / \sqrt{0.25}$. Moreover, we compute $P$ using the same formula as above, but with $\tilde{Z}$ in place of the original Gaussian instrument. We then standardize the price and then apply sigmoid to get our binary treatment. We generate datasets with $n = 2048$ samples.
 
\xhdr{Results} \cref{tab:airplane} shows results on the airplane demand benchmark. All methods except the fully Bayesian method achieve perfect validity, but \method\ produces the tightest bounds (normalized width 0.26 vs.\ 0.62--0.84 for valid baselines). The fully Bayesian method again produces overly narrow intervals (0.13), resulting in 0.90 validity. This reinforces that naively targeting the causal effects directly with exclusive KL variational inference can lead to invalid intervals. \method\ achieves 35--275$\times$ speedup.
\begin{table}[ht]
\centering
\small
\caption{Results on the airplane demand benchmark. We report validity, normalized width, and inference time (mean $\pm$ ste over seeds). Best results in validity and inference time are \textbf{bolded}; The best width for perfectly valid intervals are \underline{underlined}.}
\label{tab:airplane}
\begin{tabular}{lccc}
\toprule
\textbf{Method} & \textbf{Validity} $\uparrow$ & \textbf{Width} $\downarrow$ & \textbf{Time/1k (s)} $\downarrow$ \\
\midrule
\method\ (Ours) & \textbf{1.00 {\footnotesize $\pm$ 0.00}} & \underline{0.257 {\footnotesize $\pm$ 0.097}} & \textbf{0.44 {\footnotesize $\pm$ 0.04}} \\
\midrule
Plug-in & \textbf{1.00 {\footnotesize $\pm$ 0.00}} & 0.625 {\footnotesize $\pm$ 0.032} & {31.7 {\footnotesize $\pm$ 3.6}} \\
Whitehouse et al. & \textbf{1.00 {\footnotesize $\pm$ 0.00}} & 0.844 {\footnotesize $\pm$ 0.033} & 58.8 {\footnotesize $\pm$ 6.0} \\
Levis et al. & \textbf{1.00 {\footnotesize $\pm$ 0.00}} & 0.652 {\footnotesize $\pm$ 0.031} & 61.4 {\footnotesize $\pm$ 6.4} \\
Bayes-BP & \textbf{1.00 {\footnotesize $\pm$ 0.00}} & {0.619 {\footnotesize $\pm$ 0.021}} & 193.9 {\footnotesize $\pm$ 8.5} \\
Bayes & 0.90 {\footnotesize $\pm$ 0.10} & 0.127 {\footnotesize $\pm$ 0.002} & 251.7 {\footnotesize $\pm$ 4.3} \\
\bottomrule
\end{tabular}
\end{table}
 
\section{Additional Experiments - STAR Benchmark}
\label{appx:star-full}
 
We present additional results on the Project STAR benchmark across different treatment contrasts and outcome measures. Tables~\ref{tab:star-regaide-regular-math}--\ref{tab:star-small-regular-math} report results for: (i) Regular+Aide vs.\ Regular class comparison with math scores, (ii) Regular+Aide vs.\ Regular with reading scores, and (iii) Small vs.\ Regular with math scores. The main text presents Small vs.\ Regular with reading scores (Table~\ref{tab:star-small-regular-reading}). Across all variants, \method\ consistently achieves perfect validity while producing competitive interval widths, and maintains significant computational advantages over all baselines.
 
\begin{table*}[t]
\centering
\caption{Results on STAR benchmark comparing Regular+Aide vs.\ Regular class sizes using math scores as outcome. We report validity, normalized interval width, and inference time per 1,000 samples (mean $\pm$ standard error over 10 seeds), stratified by instrument strength. Best validity and time in \textbf{bold}; best width among perfectly valid methods is \underline{underlined}.}
\label{tab:star-regaide-regular-math}
\resizebox{\linewidth}{!}{%
\begin{tabular}{lccc|ccc}
\toprule
& \multicolumn{3}{c|}{\textbf{Weak instrument} $\rho(Z, T) \approx 0.28$} & \multicolumn{3}{c}{\textbf{Strong instrument} $\rho(Z, T) \approx 0.89$} \\
\cmidrule(lr){2-4}\cmidrule(lr){5-7}
\textbf{Method} &
\textbf{Validity} & \textbf{ Width} & \textbf{Time/1k (s)} &
\textbf{Validity} & \textbf{ Width} & \textbf{Time/1k (s)} \\
\midrule
\method\ (Ours) & \textbf{1.00 {\footnotesize $\pm$ 0.00}} & {0.193 {\footnotesize $\pm$ 0.023}} & \textbf{0.15 {\footnotesize $\pm$ 0.01}} & \textbf{1.00 {\footnotesize $\pm$ 0.00}} & \underline{0.049 {\footnotesize $\pm$ 0.002}} & \textbf{0.15 {\footnotesize $\pm$ 0.01}} \\
 
\midrule
Plug-in & \textbf{1.00 {\footnotesize $\pm$ 0.00}} & {0.651 {\footnotesize $\pm$ 0.013}} & 6.14 {\footnotesize $\pm$ 1.45}
        & \textbf{1.00 {\footnotesize $\pm$ 0.00}} & 0.097 {\footnotesize $\pm$ 0.007} & 4.03 {\footnotesize $\pm$ 0.40} \\
Whitehouse et al. & \textbf{1.00 {\footnotesize $\pm$ 0.00}} & 0.786 {\footnotesize $\pm$ 0.014} & 8.32 {\footnotesize $\pm$ 1.18}
                  & \textbf{1.00 {\footnotesize $\pm$ 0.00}} & 0.129 {\footnotesize $\pm$ 0.012} & 8.03 {\footnotesize $\pm$ 0.85} \\
Levis et al. & \textbf{1.00 {\footnotesize $\pm$ 0.00}} & 0.677 {\footnotesize $\pm$ 0.012} & 7.87 {\footnotesize $\pm$ 0.74}
             & \textbf{1.00 {\footnotesize $\pm$ 0.00}} & 0.107 {\footnotesize $\pm$ 0.010} & 8.30 {\footnotesize $\pm$ 0.97} \\
Bayes-BP & \textbf{1.00 {\footnotesize $\pm$ 0.00}} & 0.713 {\footnotesize $\pm$ 0.016} & 64.60 {\footnotesize $\pm$ 4.23} & \textbf{1.00 {\footnotesize $\pm$ 0.00}} & 0.201 {\footnotesize $\pm$ 0.010} & 67.76 {\footnotesize $\pm$ 4.34} \\
Bayes  & \textbf{1.00 {\footnotesize $\pm$ 0.00}} & \underline{0.068 {\footnotesize $\pm$ 0.001}} & 132.61 {\footnotesize $\pm$ 22.89} & \textbf{1.00 {\footnotesize $\pm$ 0.00}} & 0.144 {\footnotesize $\pm$ 0.005} & 134.27 {\footnotesize $\pm$ 14.07} \\
 
\bottomrule
\end{tabular}%
}
\end{table*}
 
\begin{table*}[t]
\centering
\caption{Results on STAR benchmark comparing Regular+Aide vs.\ Regular class sizes using reading scores as outcome. We report validity, normalized interval width, and inference time per 1,000 samples (mean $\pm$ standard error over 10 seeds), stratified by instrument strength. Best validity and time in \textbf{bold}; best width among perfectly valid methods is \underline{underlined}.}
\label{tab:star-regaide-regular-reading}
\resizebox{\linewidth}{!}{%
\begin{tabular}{lccc|ccc}
\toprule
& \multicolumn{3}{c|}{\textbf{Weak instrument} $\rho(Z, T) \approx 0.28$} & \multicolumn{3}{c}{\textbf{Strong instrument} $\rho(Z, T) \approx 0.89$} \\
\cmidrule(lr){2-4}\cmidrule(lr){5-7}
\textbf{Method} &
\textbf{Validity} & \textbf{ Width} & \textbf{Time/1k (s)} &
\textbf{Validity} & \textbf{ Width} & \textbf{Time/1k (s)} \\
\midrule
\method\ (Ours) & \textbf{1.00 {\footnotesize $\pm$ 0.00}} & {0.194 {\footnotesize $\pm$ 0.033}} & \textbf{0.15 {\footnotesize $\pm$ 0.01}} & \textbf{1.00 {\footnotesize $\pm$ 0.00}} & \underline{0.039 {\footnotesize $\pm$ 0.002}} & \textbf{0.15 {\footnotesize $\pm$ 0.01}} \\
 
\midrule
Plug-in & \textbf{1.00 {\footnotesize $\pm$ 0.00}} & 0.654 {\footnotesize $\pm$ 0.015} & 4.02 {\footnotesize $\pm$ 0.47}
        & \textbf{1.00 {\footnotesize $\pm$ 0.00}} & 0.099 {\footnotesize $\pm$ 0.006} & 4.48 {\footnotesize $\pm$ 0.82} \\
Whitehouse et al. & \textbf{1.00 {\footnotesize $\pm$ 0.00}} & 0.789 {\footnotesize $\pm$ 0.015} & 7.79 {\footnotesize $\pm$ 0.62}
                  & \textbf{1.00 {\footnotesize $\pm$ 0.00}} & 0.136 {\footnotesize $\pm$ 0.010} & 7.79 {\footnotesize $\pm$ 0.75} \\
Levis et al. & \textbf{1.00 {\footnotesize $\pm$ 0.00}} & 0.682 {\footnotesize $\pm$ 0.014} & 8.16 {\footnotesize $\pm$ 1.14}
             & \textbf{1.00 {\footnotesize $\pm$ 0.00}} & 0.108 {\footnotesize $\pm$ 0.009} & 7.27 {\footnotesize $\pm$ 0.44} \\
Bayes-BP & \textbf{1.00 {\footnotesize $\pm$ 0.00}} & 0.722 {\footnotesize $\pm$ 0.006} & 68.61 {\footnotesize $\pm$ 4.00} & \textbf{1.00 {\footnotesize $\pm$ 0.00}} & 0.208 {\footnotesize $\pm$ 0.002} & 75.25 {\footnotesize $\pm$ 15.78} \\
Bayes  & \textbf{1.00 {\footnotesize $\pm$ 0.00}} & \underline{0.072 {\footnotesize $\pm$ 0.001}} & 141.94 {\footnotesize $\pm$ 9.28} & \textbf{1.00 {\footnotesize $\pm$ 0.00}} & 0.139 {\footnotesize $\pm$ 0.008} & 155.83 {\footnotesize $\pm$ 19.42} \\
 
\bottomrule
\end{tabular}%
}
\end{table*}
 
\begin{table*}[t]
\centering
\caption{Results on STAR benchmark comparing Small vs.\ Regular class sizes using math scores as outcome. We report validity, normalized interval width, and inference time per 1,000 samples (mean $\pm$ standard error over 10 seeds), stratified by instrument strength. Best validity and time in \textbf{bold}; best width among perfectly valid methods is \underline{underlined}.}
\label{tab:star-small-regular-math}
\resizebox{\linewidth}{!}{%
\begin{tabular}{lccc|ccc}
\toprule
& \multicolumn{3}{c|}{\textbf{Weak instrument} $\rho(Z, T) \approx 0.29$} & \multicolumn{3}{c}{\textbf{Strong instrument} $\rho(Z, T) \approx 0.89$} \\
\cmidrule(lr){2-4}\cmidrule(lr){5-7}
\textbf{Method} &
\textbf{Validity} & \textbf{ Width} & \textbf{Time/1k (s)} &
\textbf{Validity} & \textbf{ Width} & \textbf{Time/1k (s)} \\
\midrule
\method\ (Ours) & \textbf{1.00 {\footnotesize $\pm$ 0.00}} & \underline{0.265 {\footnotesize $\pm$ 0.036}} & \textbf{0.17 {\footnotesize $\pm$ 0.01}} & \textbf{1.00 {\footnotesize $\pm$ 0.00}} & \underline{0.052 {\footnotesize $\pm$ 0.003}} & \textbf{0.17 {\footnotesize $\pm$ 0.02}} \\
 
\midrule
Plug-in & \textbf{1.00 {\footnotesize $\pm$ 0.00}} & {0.659 {\footnotesize $\pm$ 0.011}} & 7.99 {\footnotesize $\pm$ 0.70}
        & 0.90 {\footnotesize $\pm$ 0.10} & 0.105 {\footnotesize $\pm$ 0.007} & 7.84 {\footnotesize $\pm$ 0.80} \\
Whitehouse et al. & \textbf{1.00 {\footnotesize $\pm$ 0.00}} & 0.822 {\footnotesize $\pm$ 0.015} & 16.5 {\footnotesize $\pm$ 2.0}
                  & \textbf{1.00 {\footnotesize $\pm$ 0.00}} & 0.148 {\footnotesize $\pm$ 0.010} & 16.2 {\footnotesize $\pm$ 1.8} \\
Levis et al. & \textbf{1.00 {\footnotesize $\pm$ 0.00}} & 0.686 {\footnotesize $\pm$ 0.014} & 17.6 {\footnotesize $\pm$ 1.6}
             & \textbf{1.00 {\footnotesize $\pm$ 0.00}} & 0.116 {\footnotesize $\pm$ 0.008} & 15.9 {\footnotesize $\pm$ 2.1} \\
Bayes-BP & \textbf{1.00 {\footnotesize $\pm$ 0.00}} & 0.715 {\footnotesize $\pm$ 0.013} & 55.81 {\footnotesize $\pm$ 15.22} & \textbf{1.00 {\footnotesize $\pm$ 0.00}} & 0.223 {\footnotesize $\pm$ 0.017} & 91.46 {\footnotesize $\pm$ 13.15} \\
Bayes  & 0.60 {\footnotesize $\pm$ 0.17} & 0.078 {\footnotesize $\pm$ 0.001} & 84.06 {\footnotesize $\pm$ 28.17} & \textbf{1.00 {\footnotesize $\pm$ 0.00}} & 0.154 {\footnotesize $\pm$ 0.005} & 168.46 {\footnotesize $\pm$ 28.61} \\
 
\bottomrule
\end{tabular}%
}
\end{table*}

\section{Additional Experiments - Calibration and Sensitivity Analysis}
\label{appx:calibration-sensitivity}

The main experiments report \emph{validity} as the binary indicator that an estimated interval contains the ground-truth SATE, averaged over $10$ seeds at a single nominal credibility level. Here we further evaluate \method's interval estimator along two complementary axes: \stepindicator{i} \emph{calibration} (\cref{appx:calibration}), the empirical coverage of the credible interval as a function of the nominal level $1 - \alpha$, and \stepindicator{ii} \emph{sensitivity} (\cref{appx:sensitivity}), the behavior of coverage and interval width as we vary the sample size $n$ and the covariate dimensionality $d$. Both studies share the same trained \method\ checkpoint used for the main results, so any variation we report is attributable to the data-generating process rather than to retraining.

\subsection{Synthetic IV Data-Generating Processes}
\label{appx:calibration-DGPs}

We design three families of IV-consistent DGPs that differ in the complexity of the structural link functions. In every family, $\vect X \in \R^d$ is sampled i.i.d.\ from $\uniform(-2, 2)^d$, $Z \sim \bernoulli(0.5)$ is the binary instrument, and $U \sim \normal{0}{1}$ is a latent confounder of $T$ and $Y$ ($Z$ is exogenous in $U$). The treatment is generated as
\begin{equation}
\label{eq:appx-calib-treatment}
    T \mid \vect X, Z, U \sim \bernoulli(p_T), \quad p_T \;=\; \sigmoid\!\lpar a \cdot \tilde Z + h_T(\vect X) + \gamma_T U\rpar,
\end{equation}
where $\tilde Z = (Z - 0.5)/\sqrt{0.25}$ is the standardized instrument, $a = 2$ controls instrument strength, $\gamma_T = 1$ controls confounding, and $p_T$ is clipped to $[0.05, 0.95]$ to enforce positivity. The continuous outcome is
\begin{equation}
\label{eq:appx-calib-outcome}
    Y \;=\; \mu_T(\vect X) + \gamma_Y U + \sigma_Y \eta, \quad \eta \sim \normal{0}{1}, \quad \gamma_Y = 1, \;\; \sigma_Y = 0.5,
\end{equation}
so $Z$ does not appear in the outcome equation, satisfying the exclusion restriction in \cref{assump:iv}.

The three families differ only in the link functions $(h_T, \mu_0, \mu_1)$:
\begin{itemize}[leftmargin=*]
    \item \textbf{Linear}: $h_T(\vect x) = \vect w_T^\top \vect x$ and $\mu_t(\vect x) = \vect w_{Y0}^\top \vect x + t\beta$, with $\vect w_T, \vect w_{Y0} \sim \normal{\zeros}{I_d/d}$ and $\beta \sim \uniform(0.5, 2.0)$ (random sign). The treatment effect is homogeneous.
    \item \textbf{Polynomial}: same as Linear but with an added centered quadratic term in $h_T$, a sinusoidal $\mu_0(\vect x) = \sin(\vect w_Y^\top \vect x)$, and $\mu_1(\vect x) = \mu_0(\vect x)  + \beta + \kappa \cdot (x_{1}^2 - 4/3)$ for $\beta \sim \uniform(0.5, 1.5)$ and $\kappa \sim \uniform(0.3, 0.8)$ (random signs). The quadratic terms are mean-centered ($\E[X^2] = 4/3$ for $X \sim \uniform(-2, 2)$) so they introduce no marginal shift.
    \item \textbf{DeepNonlinear}: $h_T$, $\mu_0$, and $\mu_1$ are independent two-layer $\tanh$-MLPs with hidden width $16$ and i.i.d.\ Gaussian weights, producing fully nonlinear treatment-assignment, baseline, and causal effect surfaces.
\end{itemize}
For each generated dataset, we treat the SATE $\lambda(\data) = \frac{1}{n}\sum_{i=1}^n [\mu_1(\vect x_i) - \mu_0(\vect x_i)]$ as the ground truth, matching the SATE convention in \cref{eq:sample-ate}. We use the same trained \method\ checkpoint as in \cref{sec:results}; continuous outcomes are handled at inference time via the discretization of \citet[\S 6]{levis2025}, as in our other continuous-outcome benchmarks.

\subsection{Calibration Curves}
\label{appx:calibration}

\xhdr{Setup} For each DGP, we generate $K = 100$ datasets at $n = 1024$ and $d = 5$ from independent seeds. For each dataset we run a single forward pass of the trained model to obtain the marginal posterior approximation $q_{\theta^\star}(\cdot \mid \data)$, and read off the credible interval at $18$ nominal levels $1 - \alpha \in \{0.01, \ldots, 0.995\}$ as $[Q_{\alpha/2}, Q_{1-\alpha/2}]$. Empirical coverage at level $1 - \alpha$ is the fraction of the $K$ datasets whose sample SATE lies inside the predicted interval, and is the multi-seed analogue of the validity metric used in \cref{sec:results}.

\xhdr{Results} \cref{fig:calibration} reports empirical coverage versus nominal credibility, with $\pm 1$ standard error bands. Across all three DGPs, empirical coverage \emph{exceeds} the nominal level for every credibility level. The model is also more conservative on the more nonlinear DGPs ($\text{DeepNonlinear} \succeq \text{Polynomial} \succeq \text{Linear}$ at every credibility between $0.4$ and $0.9$). This observation is consistent with the inclusive KL loss that induces conservative over-coverage.

\begin{figure}[t]
    \centering
    \captionsetup{font=small}
    \includegraphics[width=0.45\linewidth]{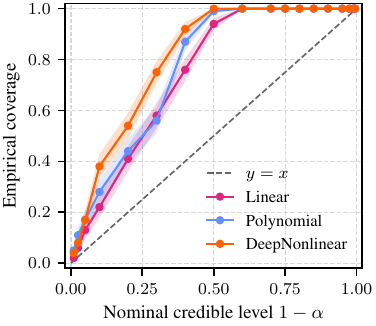}
    \caption{\textbf{Calibration curves on three synthetic IV DGPs.} Empirical coverage versus nominal credibility level $1 - \alpha$ over $K = 100$ datasets per DGP at $n = 1024$, $d = 5$. Bands show $\pm 1$ standard error. \method\ over-covers in the regime that matters for inference and approaches the diagonal smoothly at very narrow nominal levels --- the expected signature of an inclusive KL-trained marginal posterior on partial identification.}
    \label{fig:calibration}
\end{figure}

\subsection{Sensitivity to Sample Size and Dimensionality}
\label{appx:sensitivity}

\xhdr{Setup} For each DGP we run two one-dimensional sweeps at fixed $\alpha = 0.1$ (i.e., $90\%$ nominal credible intervals): a sample-size sweep with $n \in \{256, 512, 1024, 2048, 4096\}$ at $d = 5$, and a dimensionality sweep with $d \in \{2, 4, 8, 16, 32\}$ at $n = 1024$. Each cell uses $K = 30$ independent seeds drawn disjointly from those in \cref{appx:calibration}. We report empirical coverage and the mean normalized interval width $(\hat{\upperbound} - \hat{\lowerbound})/(\max Y - \min Y)$, matching the informativeness convention in \cref{sec:experiments}.

\xhdr{Results} \cref{fig:sensitivity} shows the two sweeps. Coverage stays at $1.00$ in every cell of both sweeps and across all three DGPs, consistent with the conservative over-coverage of the $90\%$ interval characterized in \cref{appx:calibration}; coverage thus carries no signal at this nominal level, and the interesting axis is interval \emph{width}. As $n$ grows from $256$ to $4096$ the mean normalized width contracts by $\sim 20\%$ (from $\sim 0.45$ to $\sim 0.36$), with all three DGPs tracking each other closely. As $d$ grows from $2$ to $32$ the width inflates by $\sim 30\%$ (from $\sim 0.37$ to $\sim 0.49$).
\begin{figure}[t]
    \centering
    \captionsetup{font=small}
    \includegraphics[width=0.85\linewidth]{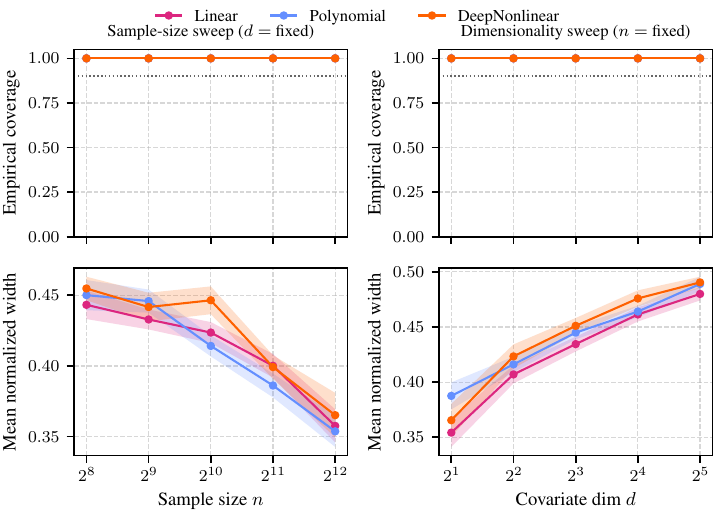}
    \caption{\textbf{Sensitivity to sample size and dimensionality} at $\alpha = 0.1$ ($90\%$ credible intervals). \emph{Top row}: empirical coverage; \emph{bottom row}: mean normalized interval width. Bands show $\pm 1$ standard error over $K = 30$ seeds per cell. Coverage saturates at $1.00$ across the entire grid; widths shrink monotonically with $n$ and grow with $d$, as expected.}
    \label{fig:sensitivity}
\end{figure}

\end{document}